\documentclass[letterpaper]{article} 
\usepackage{aaai2027}  
\usepackage[hyphens]{url}  
\usepackage{graphicx} 
\graphicspath{{figures/}{../figures/}}
\usepackage{natbib}  
\usepackage{caption} 
\usepackage{algorithm}
\usepackage{algorithmic}
\usepackage{amsmath,amssymb,amsthm}
\usepackage{paralist}

\usepackage{newfloat}
\usepackage{listings}
\DeclareCaptionStyle{ruled}{labelfont=normalfont,labelsep=colon,strut=off} 
\floatstyle{ruled}
\newfloat{listing}{tb}{lst}{}
\floatname{listing}{Listing}

\usepackage{booktabs}

\usepackage{hyperref} 

\title{Quantitative Analysis of $\omega$-Regular Robust MDPs}
\author {
    Ali Asadi\textsuperscript{\rm 1},
    Krishnendu Chatterjee\textsuperscript{\rm 1},
    Ehsan Kafshdar Goharshady\textsuperscript{\rm 1},
    Mehrdad Karrabi\textsuperscript{\rm 1},
    Alipasha Montaseri\textsuperscript{\rm 1},
    Ali Shafiee\textsuperscript{\rm 1}
}
\affiliations {
    \textsuperscript{\rm 1}Institute of Science and Technology Austria\\
    ali.asadi@ist.ac.at, krichnendu.chatterjee@ist.ac.at, ehsan.goharshady@ist.ac.at, mehrdad.karrabi@ist.ac.at, alipasha.montaseri@ist.ac.at, ali.shafiee@ist.ac.at
}

\newtheorem{theorem}{Theorem}              
\newtheorem{lemma}[theorem]{Lemma}
\newtheorem{proposition}[theorem]{Proposition}
\newtheorem{corollary}[theorem]{Corollary}
\newtheorem{definition}[theorem]{Definition}

\newtheorem{remark}[theorem]{Remark}

\usepackage{thmtools}

\newcommand{\states}{\mathcal{S}}
\newcommand{\actions}{\mathcal{A}}
\newcommand{\trans}{\delta}
\newcommand{\mdp}{M}
\newcommand{\rmdp}{\mathcal{M}}
\newcommand{\agentpol}{\sigma}
\newcommand{\envpol}{\tau}
\newcommand{\uncert}{\mathcal{P}}
\newcommand{\allagentpols}{\Sigma}
\newcommand{\allenvpols}{\Gamma}
\newcommand{\sinit}{{s_0}}
\newcommand{\spec}{\varphi}
\newcommand{\val}{\boldsymbol{v}}
\newcommand{\Inf}{\mathsf{Inf}}
\newcommand{\coloring}{C}
\newcommand{\parity}{\mathsf{Parity}}
\newcommand{\reach}{\mathsf{Reach}}
\newcommand{\safe}{\mathsf{Safe}}
\newcommand{\target}{\mathcal{T}}
\newcommand{\rmc}{\mathcal{R}}
\newcommand{\AS}{\mathsf{AS}}

\newcommand{\cross}{\star}

\newcommand{\vertexgame}{\mathcal{G}}                 
\newcommand{\vertices}{\mathsf{Vert}}                 
\newcommand{\supp}{\mathsf{supp}}                     
\newcommand{\vc}[2]{\mathsf{VC}^{#1}(#2)}             
\newcommand{\tightacts}[2]{\actions^{#1}(#2)}         
\newcommand{\tightset}[3]{\uncert^{#1}(#2,#3)}        
\newcommand{\ProfSwitch}{\mathsf{Improve}}
\newcommand{\StratImp}{\mathsf{PI}}

\newcommand{\asoddsat}{\mathsf{OddEC}}
\newcommand{\mecdecomp}{\mathsf{MecDecomp}}
\newcommand{\uvector}{\boldsymbol{u}}
\newcommand{\recurrent}{\mathcal{RC}}

\nocopyright
\begin{document}

\maketitle

\begin{abstract}
Robust Markov Decision Processes (RMDPs) generalize classical MDPs by
allowing uncertainty in transition probabilities and optimizing against
their worst-case realization. We consider $(s,a)$-rectangular RMDPs with
\emph{linearly defined} uncertainty sets and study parity objectives, which are a
canonical representation of $\omega$-regular objectives. An uncertainty set
is linearly defined if it is described by linear inequalities over the
transition distribution together with auxiliary variables, which captures the
standard $L_1$ and $L_\infty$ balls as well as general polytopic uncertainty sets. The quantitative
value is the supremum, over all agent policies, of the satisfaction
probability guaranteed against the adversarial environment. Previous
work studied the qualitative analysis, namely the almost-sure (resp. positive) problem that asks whether a single agent policy guarantees
satisfaction with probability one (resp. positive probability) against every environment policy. In this work, we solve the exact quantitative problem.
Our contributions are threefold.
First, we show that both the agent and the environment admit pure
memoryless optimal policies.
Second, we give a polynomial-time algorithm for quantitative parity on
linearly defined robust Markov chains and use it as a subroutine in a
policy-iteration algorithm for RMDPs. The algorithm combines
quantitative one-step improvements with qualitative almost-sure
improvements. Finally, we report experiments comparing our
approach with the explicit reduction to stochastic games. 
\end{abstract}


\section{Introduction} \label{sec:intro}

\paragraph{Robust MDPs.}
Markov Decision Processes (MDPs) are a fundamental framework for
reasoning, decision making, and stochastic planning~\cite{puterman94}.
In a classical MDP, an agent chooses actions and the
next state is determined stochastically according to a known transition function. However,
transition probabilities are often estimated from data, and the exact
transition function is therefore not known. Robust Markov Decision
Processes (RMDPs) address this issue by associating every state-action
pair with an uncertainty set of possible transition
distributions~\cite{iyengar2005robust,nilim2005robust,wiesemann2013}. After observing
the action of the agent, an adversarial environment chooses a distribution
from the corresponding uncertainty set, and the goal of the agent is to
optimize against all such choices. This setting is known as
$(s,a)$-rectangular uncertainty. In this work, we consider \emph{linearly
defined} uncertainty sets, i.e.~sets described by rational linear inequalities
over the transition distribution together with auxiliary variables that are
existentially quantified. The auxiliary variables keep the description
succinct: the standard $L_1$ and $L_\infty$ balls around a nominal transition
distribution are linearly defined with linearly many inequalities, whereas an
$L_1$ ball has exponentially many facets in the transition probabilities alone.
Since every uncertainty set lies in the probability simplex, linearly defined
sets are precisely the convex polytopes, so this model subsumes general
polytopic uncertainty sets while admitting far more succinct descriptions.

\paragraph{Parity and $\omega$-Regular Objectives.}
Most of the RMDP literature considers quantitative objectives such as
discounted sum or long-run average reward. These objectives provide
performance guarantees, but they do not express the
logical correctness requirements that arise in safety-critical systems, e.g., an autonomous system is required to avoid unsafe configurations, eventually complete a task, or repeatedly provide a service. Such reachability, safety, and
liveness requirements are naturally expressed by $\omega$-regular
objectives. Parity objectives are a canonical representation of this
class. An $\omega$-regular objective can be represented by a finite
deterministic parity automaton. Combining this automaton with an RMDP yields a product
RMDP with a parity objective~\cite{model-checking-book}.
Hence, RMDPs with parity objectives provide a flexible and broad framework
for the analysis of probabilistic systems with logical objectives.

\paragraph{Quantitative Analysis.}
Given an RMDP and a logical objective, the value of a state is the
maximal satisfaction probability guaranteed by the agent 
against every environment choice. Qualitative analysis considers two 
questions. The almost-sure (resp. positive) problem asks whether there exists an agent policy
that guarantees satisfaction with probability one (resp. greater than zero) against every
environment choice. However, these qualitative analyses do not distinguish among policies when
no almost-sure guarantee exists. Hence, quantitative analysis is important
for finding optimal policies and computing the values.
The work of~\citet{AsadiCGKS26} gives algorithms for the positive and
almost-sure problems: polynomial time for reachability and safety objectives,
and quasi-polynomial time and polynomial space for all parity objectives.
However, the quantitative analysis remains open.

\paragraph{Contributions.} 
We solve the quantitative parity problem for RMDPs with linearly defined
uncertainty sets. Our contributions are threefold.
\begin{enumerate}
    \item We first consider robust Markov chains (RMCs), i.e., RMDPs in
    which the agent has a single action. For reachability objectives, we
    formulate a polynomial-size linear program. We then combine this formulation with an analysis
    of recurrent components to solve parity objectives, which
    yields a polynomial-time algorithm for the quantitative parity problem.

    \item We then present a policy-iteration algorithm for RMDPs. We
    establish that both the agent and the environment admit stationary
    memoryless optimal policies.
    Each iteration evaluates the RMC induced by the
    current agent policy and applies either a quantitative one-step
    improvement or a qualitative almost-sure improvement. 
    We show that the algorithm terminates
    with the exact value vector and optimal agent and environment policies.

    \item Finally, we implement our algorithms for reachability and
    parity objectives under $L_1$ and $L_\infty$ uncertainty. We compare
    our approach with the explicit reduction to a turn-based stochastic
    game. Our approach scales substantially better on the Garnet and
    Inventory benchmarks, whose branching factor grows with the instance
    size, whereas the explicit reduction is faster on the 
    Frozen Lake benchmark, which has a small branching factor.
\end{enumerate}

\paragraph{Related Work.}
Our policy-iteration algorithm is, in spirit, close to the
strategy-improvement method for stochastic parity and Rabin
games~\cite{chatterjee2005algorithms,chatterjee2006strategy}, but in our robust
setting the environment ranges over a continuum of distributions in an
uncertainty polytope rather than over the finitely many vertices of a stochastic
game. Our improvement therefore works directly with the linearly defined
uncertainty sets (through the value classes, tight actions, and tight faces of
Section~\ref{sec:rmdp}) and never constructs the exponentially larger game. The
closest prior work on logical objectives for RMDPs is~\citet{AsadiCGKS26}, which
solves the positive and almost-sure problems using uncertainty-set oracles. Our
qualitative improvement step invokes their almost-sure parity procedure as a
sub-routine, whereas the exact quantitative problem we solve here remained open.
An extended discussion of related work is deferred to Appendix~\ref{app:related}.

\section{Preliminaries and Model Description}

\paragraph{Notation.} For any finite set $A$ we use $\Delta(A)$ and $2^A$ to denote the set of all distributions over $A$ and the power-set of~$A$ respectively. We use $[n]$ as shorthand for $\{0,\dots,n\}$.

\paragraph{MDPs and MCs.}
\emph{A Markov decision process (MDP)} is a tuple $\mdp = (\states,\actions,\trans)$ where $\states$ is a finite set of states, $\actions$ is a finite set of actions and $\trans\colon \states \times \actions \to \Delta(\states)$ specifies the transition probabilities. A \emph{Markov chain (MC)} is an MDP whose action set is a singleton, i.e. $|\actions|=1$.

The semantics of MDPs are defined by \emph{policies}. An \emph{agent policy} is $\agentpol \colon (\states\times\actions)^*\times \states \to \Delta(\actions)$ that maps a history of states and actions to a distribution over actions. Given such a policy a token is placed on an initial state $s_0 \in \states$ and at the $i$-th step, the agent draws an action $a_i \sim \agentpol(s_0,a_0,\dots,s_{i-1},a_{i-1},s_i)$ and the next state $s_{i+1}$ is sampled from $\trans(s_i,a_i)$. Fixing an initial state $\sinit$, the cylinder construction of~\cite{model-checking-book} generates a probability distribution $\Pr^{\agentpol}_{\sinit}$ over infinite paths in the MDP. The semantics of MCs are defined analogously, but since there is no choice of actions, the next state is sampled directly from the transition probabilities.

\paragraph{RMDPs and RMCs.}
Robust Markov Decision Processes (RMDPs) extend MDPs by allowing uncertainty in the transition probabiliaties. Formally, an RMDP is a tuple $\rmdp = (\states,\actions,\uncert)$ where $\states$ and $\actions$ are as before, but $\uncert\colon \states \times \actions \to 2^{\Delta(\states)}$ maps each state-action pair to a set of possible transition distributions, called the uncertainty set. A Robust Markov Chain (RMC) is an RMDP where $|\actions|=1$, hence an RMC $\rmc$ is defined by a pair $(\states,\uncert)$. 


The semantics of RMDPs are defined by a pair of policies: an agent policy $\agentpol$ as before, and an environment policy $\envpol \colon (\states\times\actions)^* \times (\states\times\actions) \to \Delta(\states)$ that maps a history of states and actions to a distribution over next states where for every history $h \in (\states\times\actions)^*$ and $(s_i,a_i) \in \states \times \actions$, it holds that $\envpol(h,s_i,a_i) \in \uncert(s_{i},a_{i})$. As with MDPs, fixing an initial state $\sinit$ and a policy profile $(\agentpol,\envpol)$ in an RMDP $\rmdp$ induces a probability distribution $\Pr^{\agentpol,\envpol}$ over infinite paths in $\rmdp$. The semantics of RMCs are defined analogously.

\paragraph{Policies.} 
An agent policy $\agentpol$ is \emph{deterministic} if it maps every history to a single action and \emph{memoryless} if it only depends on the current state, i.e. $\agentpol(h,s) = \agentpol(h',s)$ for all histories $h,h'$ and state $s \in \states$. A \emph{positional} policy is both deterministic and memoryless. Conventionally, we assume positional policies are $\states \to \actions$ functions. We denote the set of all agent policies by $\allagentpols$ and the set of all environment policies by $\allenvpols$. Fixing a positional policy $\agentpol$ in an RMDP $\rmdp$, induces an RMC $\rmdp^\agentpol$ where the only available action at each state $s$ is $\agentpol(s)$ and the uncertainty set is defined as $\uncert^\agentpol(s) = \uncert(s,\agentpol(s))$.

\begin{remark}
    This work is concerned with $(s,a)$-rectangular RMDPs whose semantics are defined above: the environment observes the action taken by the agent and then chooses a distribution from the corresponding uncertainty set.
\end{remark}

\paragraph{Objectives and Values.}
An objective $\spec$ is a measurable set of infinite paths in $\rmdp$. Given an agent policy $\agentpol$ and an environment policy $\envpol$, we denote by $\Pr^{\agentpol,\envpol}_{\rmdp,s}[\spec]$ the probability that a path starting from state $s$ satisfies the objective $\spec$. The \emph{value} of an objective $\spec$ in an RMDP $\rmdp$ from $s$ is
$$\val^{\rmdp}_s(\spec) = \sup_{\agentpol \in \allagentpols} \inf_{\envpol \in \allenvpols} \Pr\nolimits^{\agentpol,\envpol}_{s}[\spec].$$
Values are defined analogously for RMCs.

We focus on \emph{parity} objectives. A parity objective $\parity(\coloring)$ is defined with respect to a coloring function $\coloring \colon \states \to [d]$. Particularly, $\parity(\coloring)$ is defined as the set of all infinite paths $\pi$ such that the maximum color that appears infinitely often in $\pi$ is even, i.e.
$$\parity(\coloring) = \{\pi \in \states^\omega \mid \max\{\coloring(s) \mid s \in \Inf(\pi)\} \text{ ~is~ even}\},$$
where $\Inf(\pi)$ denotes the set of states that occur infinitely often in $\pi$.
It is a classical result that every $\omega$-regular objective, including reachability, safety and liveness, can be reduced to a parity objective \cite{model-checking-book}.


\paragraph{Model Assumption.}
A \emph{linearly defined} uncertainty set $\uncert(s,a)$ is the projection of a polyhedron onto the probability coordinates, i.e.~there exist a matrix $A$ and a vector $b$, of dimensions compatible with $k$ auxiliary variables, where:
$$\uncert(s,a) = \Big\{ x \in \Delta(\states) \;\Big|\; \exists y \in \mathbb{R}^k \colon A \begin{bmatrix} x \\ y \end{bmatrix} \le b \Big\}.$$
Note that the auxiliary variables $y$ are existentially quantified, so a linearly defined set need not be cut out by linear inequalities over $x$ alone. This is strictly more expressive in terms of succinctness: an $L_1$ ball around a nominal distribution is linearly defined using only $O(|\states|)$ inequalities and auxiliary variables, whereas over $x$ alone it has exponentially many facets (see Appendix~\ref{app:uncertainty}). Finally, since $\uncert(s,a) \subseteq \Delta(\states)$ is bounded, every linearly defined uncertainty set is a convex polytope. Throughout this paper, we assume that every uncertainty set is linearly defined. This assumption is satisfied by the standard $L_1$ and $L_\infty$ balls (See Appendix~\ref{app:uncertainty}), as well as by the more general polytopic uncertainty sets.

\paragraph{Problem Statement.}
The main problem statement in this paper is summarized as follows:

\begin{center}
\fbox{%
\begin{minipage}{0.95\columnwidth}
Given a linearly defined RMDP $\rmdp$ and a parity objective $\parity(\coloring)$, the \emph{quantitative parity problem} asks to compute the value $\val^{\rmdp}_s(\parity(\coloring))$ for every state $s \in \states$.
\end{minipage}%
}
\end{center}

In Section~\ref{sec:rmc}, we present an algorithm for solving the quantitative parity problem in RMCs that runs in polynomial time. We will than use the RMC algorithm as a sub-procedure to solve the quantitative parity problem in RMDPs in Section~\ref{sec:rmdp}.
\section{Quantitative RMC Analysis Sub-Procedure} \label{sec:rmc}

We present a polynomial-time algorithm for solving the quantitative safety and parity objectives in linearly defined RMCs $\rmc = (\states,\uncert)$ where the environment's goal is to minimize the probability that the objective is satisfied. This is then used in the algorithm for solving RMDPs as a sub-procedure. In what follows, we first consider the case of safety objectives (Section~\ref{sec:RMC:reachability}) and present an LP-based algorithm for solving them. We then use the method for solving safety objectives and present a solution for the quantitative parity problem (Section~\ref{sec:RMC:parity}).

\subsection{Safety} \label{sec:RMC:reachability}
Given a set $\target \subseteq \states$, the \emph{safety} objective $\safe(\target)$ is the set of infinite paths that never leave $\target$. The environment is adversarial and \emph{minimizes} the probability of remaining safe, so the value at a state $s$ is $\val^{\rmc}_s(\safe(\target)) = \inf_{\envpol} \Pr\nolimits^{\envpol}_s[\safe(\target)]$. We compute this value for every $s$ with a single linear program, exploiting the duality of safety and reachability together with the flow-like characterization of reachability probabilities.


\paragraph{Pre-Processing.} We first dispose of the states whose safety value is $0$ or $1$. We use the polynomial-time method of~\cite{AsadiCGKS26} to compute the sets $\states^{=0}$ and $\states^{=1}$ of states with safety value $0$ and $1$, respectively. Our pre-processing step makes all states in $\states^{=0} \cup \states^{=1}$ absorbing. This leaves the states $\states^{?} = \states \setminus (\states^{=0} \cup \states^{=1})$, whose values are to be computed.

\paragraph{Flow-like Characterization.} Fix a source state $\sinit \in \states^{?}$ and picture injecting one unit of probability mass there. As the play unfolds, this mass is routed through the state space along the transition probabilities until it is absorbed in $\states^{=0} \cup \states^{=1}$. The adversarial environment moves this mass so as to \emph{minimize} the amount absorbed in $\states^{=1}$ and \emph{maximize} the amount absorbed in $\states^{=0}$. Recording, for each state $s$, the expected number of visits $x_s$, and for each edge $(s,s')$ the expected number of traversals $f_{s,s'}$, these quantities behave exactly like a conserved flow. They satisfy the following linear properties, which we later assemble into the LP:
\begin{compactitem}
  \item \textbf{Non-negativity.} No state is visited and no edge is traversed a negative number of times:
  \[
      \mathsf{NNeg} \;\equiv\; \bigwedge_{s \in \states^{?},\, s' \in \states} \big( x_s \ge 0 \;\wedge\; f_{s,s'} \ge 0 \big).
  \]

  \item \textbf{Conservation.} The expected number of visits to a state equals the mass injected there plus the total flow entering it, which is Kirchhoff's conservation law:
  \[
      \mathsf{Cons} \;\equiv\; \bigwedge_{s \in \states^{?}} \Big( x_s \;=\; \mathbf{1}[\,s = \sinit\,] \;+\; \sum_{s' \in \states^{?}} f_{s', s} \Big).
  \]

  \item \textbf{Splitting.} Each visit to a state is followed by exactly one outgoing edge, so the visits to a state are split among its outgoing flow:
  \[
      \mathsf{Split} \;\equiv\; \bigwedge_{s \in \states^{?}} \Big( \sum_{s' \in \states} f_{s, s'} \;=\; x_s \Big).
  \]

  \item \textbf{Admissibility.} The proportions in which the flow leaves a state form a distribution $\boldsymbol{f}_s / x_s$, where $\boldsymbol{f}_s = (f_{s,s'})_{s' \in \states}$, that the environment is allowed to pick, i.e. $\boldsymbol{f}_s / x_s \in \uncert(s)$. This membership is, at face value, nonlinear: it divides the flow variables by the visit variable. However, by the Assumption $\uncert(s) = \{p : A_s\, p \le b_s\}$, so the constraint reads $A_s (\boldsymbol{f}_s / x_s) \le b_s$; multiplying both sides by $x_s \ge 0$ preserves the inequality and clears the denominator, yielding the equivalent linear constraint
  \[
      \mathsf{Adm} \;\equiv\; \bigwedge_{s \in \states^{?}} \big( A_s\, \boldsymbol{f}_s \;\le\; x_s\, b_s \big).
  \]
  This is jointly linear in the flow variables $\boldsymbol{f}_s$ and the visit variable $x_s$, since $A_s$ and $b_s$ are constants.
\end{compactitem}
Since $\states^{=1}$ is absorbing, the probability of reaching $\states^{=1}$ from $\sinit$ is simply the total flow that crosses into $\states^{=1}$. Hence, the environment minimises $\sum_{\substack{s \in \states^{?} \\ w \in \states^{=1}}} f_{s, w}$, which is therefore the objective of the LP. The following lemma (proved in Appendix~\ref{app:rmc:reach}) formalises the above discussion and establishes the correctness of the LP:

\begin{lemma}\label{lem:rmc:reachability}
Given a linearly defined RMC $\rmc$, initial state $\sinit \in \states^{?}$ and with safe set $\target \subseteq \states$, let $\val^*$ be the solution to the following linear program:
\begin{equation}\label{eq:rmc:reach:lp}\tag{$\cross$}
    \begin{aligned}
        \min_{x,\, f} \quad & \sum_{\substack{s \in \states^{?},\, w \in \states^{=1}}} f_{s, w} \\[2pt]
        \text{s.t.} \quad & \mathsf{NNeg} \,\wedge\, \mathsf{Cons} \,\wedge\, \mathsf{Split} \,\wedge\, \mathsf{Adm},
    \end{aligned}
\end{equation}
then $\val^{\rmc}_{\sinit}(\safe(\target)) = \val^*$.
\end{lemma}

Building on this, we obtain that the safety values of an RMC, together with optimal environment policies, can be computed in polynomial time.

\begin{theorem}\label{thm:rmc:safety-poly}
Given a linearly defined RMC $\rmc$ with safety objective $\safe(\target)$, for each fixed $s$, the value $\val^{\rmc}_{s}(\safe(\target))$, together with a positional environment policy that attains it, can be computed in time polynomial in the size of $\rmc$.
\end{theorem}

\subsection{Parity} \label{sec:RMC:parity}

In this section, we solve the quantitative parity problem by reducing it to a safety computation. Recall that the environment is adversarial and \emph{minimizes} the probability of satisfying $\parity(\coloring)$. We first show a polynomial-time algorithm that computes the set $W_{\mathrm{odd}}$ of states lying in an end-component whose dominant color is odd implying that the environment can violate $\parity(\coloring)$ with probability $1$. We then show that the parity value at any state equals the value of $\safe(\states \setminus W_{\mathrm{odd}})$, i.e., the environment's goal is to minimize the probability of not reaching $W_{\mathrm{odd}}$. This value is then computed by the safety variant of the LP of Section~\ref{sec:RMC:reachability}.

\subsubsection{Almost-Sure Violation of Parity.}
We compute the set $W_{\mathrm{odd}}$ of states lying in an end-component from which the environment can violate $\parity(\coloring)$ with probability $1$. We note that the algorithm of~\cite{AsadiCGKS26} can be used to compute this qualitative information too, however, their method requires super-polynomial time in the worst case, while our goal is to provide a polynomial-time algorithm. 

We first define the notion of \emph{maximal end-components} (MECs) in RMCs, which are the analogue of MECs in MDPs: 

\begin{definition}\label{def:ec}
For a set $U \subseteq \states$ and a state $s \in U$, let
$\uncert_U(s) = \{\, p \in \uncert(s) \;:\; \supp(p) \subseteq U \,\}$
be the set of distributions at $s$ with which the environment can keep the play inside $U$.
A set $U \subseteq \states$ is an \emph{end-component} (EC) of an RMC $\rmc$ if the following conditions hold:
\begin{compactenum}
    \item For every state $s \in U$, we have $\uncert_U(s) \neq \emptyset$.
    \item The graph induced by $\uncert_U$ on $U$ is strongly connected.
\end{compactenum}
Furthermore, $U$ is a maximal EC (MEC) if there is no end-component $U' \subseteq \states$ such that $U \subset U'$.
\end{definition}

Intuitively, whenever the play enters a (maximal) EC, the environment can ensure that the play remains in the end-component forever. The environment's goal would then be to reach an end-component in which the parity objective can be violated almost-surely.

The $\mecdecomp$ procedure in Algorithm~\ref{algo:mec-decomp} of Appendix~\ref{app:rmc:mec-decomp} takes an RMC $\rmc$ as input and returns its MECs in polynomial time. The algorithm is based on the standard algorithm for computing MECs in MDPs~\cite{ChatterjeeH11}, and utilizes the polytopic structure of the uncertainty sets.

Algorithm~\ref{algo:as-odd-sat} illustrates the $\asoddsat$ procedure which utilizes $\mecdecomp$ to compute the set $W_{\mathrm{odd}}$ as the union of those ECs where the environment can almost-surely violate $\parity(\coloring)$ in them. Since the play can be confined to a MEC and visits all of its states infinitely often, the dominant color the environment can enforce in a MEC is its highest color $m_i$: if $m_i$ is odd, the environment stays in the MEC forever and makes the dominant color odd, so the whole MEC violates the objective and is added to $W_{\mathrm{odd}}$; otherwise the environment must avoid the even color $m_i$ becoming dominant, so we discard the states of color $m_i$ and recurse on the rest. 

\begin{algorithm}[t]
\caption{$\asoddsat(B)$}
\label{algo:as-odd-sat}
\begin{algorithmic}[1]
\REQUIRE A subset $B \subseteq \states$; the top-level call uses $B = \states$.
\ENSURE Union of ECs in $B$ whose dominant color is odd. 
\STATE $M_1, \dots, M_k \gets \mecdecomp(B)$
\STATE $W \gets \emptyset$
\FOR{$i = 1, \dots, k$}
    \STATE $m_i \gets \max_{s \in M_i} \coloring(s)$
    \STATE \algorithmicif\ $m_i$ is odd \algorithmicthen\ $W \gets W \cup M_i$
        \STATE \algorithmicelse\ $W \gets W \cup \asoddsat(M_i \setminus \coloring^{-1}(m_i))$
\ENDFOR
\STATE \textbf{return} $W$
\end{algorithmic}
\end{algorithm}

The following lemma (proved in Appendix~\ref{app:rmc:parity}) establishes the correctness and complexity of the $\asoddsat$ procedure.

\begin{lemma} \label{thm:rmc:almost-sure-parity}
    $\asoddsat$ computes $W_{\mathrm{odd}}$ in polynomial time.
\end{lemma}

\subsubsection{Quantitative Satisfaction of Parity.}
Having computed $W_{\mathrm{odd}}$, we reduce the quantitative parity problem to the safety computation of Section~\ref{sec:RMC:reachability}. From any state in $W_{\mathrm{odd}}$ the environment violates $\parity(\coloring)$ almost-surely, and, conversely, under every environment policy, almost every play settles in an EC. 
So, under every environment policy, the probability of violating
$\parity(\coloring)$ is at most the probability of reaching $W_{\mathrm{odd}}$,
and the environment attains it by playing optimally for reachability of
$W_{\mathrm{odd}}$ and then enforcing an odd dominant color. The following lemma establishes this reduction, and is proved in Appendix~\ref{app:rmc:parity}.

\begin{lemma} \label{lem:rmc:parity}
    Given a linearly defined RMC $\rmc = (\states, \uncert)$ it holds that $\val^{\rmc}_{s}(\parity(\coloring)) = \val^{\rmc}_{s}\big(\safe(\states \setminus W_{\mathrm{odd}})\big)$.
\end{lemma}
The right-hand side is exactly the value computed by the LP $(\cross)$ of Section~\ref{sec:RMC:reachability}, instantiated with unsafe set $U = W_{\mathrm{odd}}$. This yields the parity values, and optimal positional environment policies, in polynomial time:
\begin{theorem} \label{thm:rmc:parity-poly}
    Given a linearly defined RMC $\rmc$ with parity objective $\parity(\coloring)$, for all $s \in \states$ the value $\val^{\rmc}_{s}(\parity(\coloring))$, together with a positional environment policy that attains it, can be computed in time polynomial in the size of $\rmc$.
\end{theorem}
\section{Quantitative Analysis of $\omega$-Regular RMDPs} \label{sec:rmdp}

Algorithm~\ref{alg:si} presents our policy-iteration algorithm for the quantitative parity
problem on linearly defined RMDPs. The algorithm ranges
over \emph{positional} agent policies and, in each iteration, tries to improve
the current policy by calling the $\ProfSwitch$ procedure of Algorithm~\ref{alg:switch}. 
Throughout this section we fix an RMDP $\rmdp$ and a parity objective $\parity(\coloring)$. In what follows we describe the algorithm and its ingredients, and state its correctness and complexity guarantees. The proofs are deferred to Appendix~\ref{app:rmdp}.



\paragraph{Positional Determinacy.}
Before diving into the algorithm, we establish that for parity objectives, positional optimal policies exist for both the agent and the environment. 
Fixing a positional agent policy $\agentpol$ induces the RMC
$\rmdp^{\agentpol}$ whose only action at each state $s$ is
$\agentpol(s)$ and whose uncertainty set is
$\uncert^{\agentpol}(s)=\uncert(s,\agentpol(s))$; we write
$\val^{\agentpol}$ for its value vector,
$\val^{\agentpol}_s=\inf_{\envpol}\Pr^{\agentpol,\envpol}_s[\spec]$. Our
analysis rests on the following determinacy result, which is where
polytopic properties of the uncertainty sets are used.

\begin{restatable}{proposition}{polytopicprop}
\label{prop:polytopic}
Let $\rmdp$ be a linearly defined RMDP with parity objective $\spec=\parity(\coloring)$. There exists a positional agent policy $\agentpol^\ast$
and a positional environment policy $\envpol^\ast$ where
  $\inf_{\envpol}\Pr^{\agentpol^\ast,\envpol}_s[\spec] = \sup_{\agentpol}\Pr^{\agentpol,\envpol^\ast}_s[\spec]
  =\val^{\rmdp}_s(\spec)$ for every $s$.
\end{restatable}

The proof uses the reduction introduced in~\cite{chatterjee2024solving} to reduce $\rmdp$ to a finite turn-based stochastic parity game 
in which the environment's choice is replaced by an adversarial choice of a corner of the uncertainty set.
The reduction is value-preserving and maps positional strategies of the game back to positional policies of $\rmdp$, so positional determinacy of stochastic games implies the proposition. The number of corners, and hence the size of the game, may however be exponential, so this route requires exponential space; our policy iteration algorithm below instead runs in polynomial space. Combined with the polynomial-time RMC solver of Section~\ref{sec:rmc}, Proposition~\ref{prop:polytopic} also settles the complexity of the decision problem.
\begin{corollary} \label{cor:np-conp}
    Given a linearly defined RMDP $\rmdp$, a state $s \in \states$ and a rational $\lambda \in [0,1]$, the problem of deciding whether $\val^{\rmdp}_s(\spec) \geq \lambda$ is (i)~in $\mathit{NP} \cap \mathit{coNP}$ and (ii)~at least as hard as the analogous problem in turn-based stochastic parity games.
\end{corollary}
Hardness implies that a polynomial-time algorithm here would also solve turn-based stochastic parity games, a long-standing open problem.

\paragraph{Policy Iteration.}
The policy iteration algorithm $\StratImp$ (Algorithm~\ref{alg:si}) starts from an arbitrary positional agent policy and iteratively calls the $\ProfSwitch$ procedure (Algorithm~\ref{alg:switch}) to improve it. Once the policy can not be improved further, the algorithm returns it together with its value vector. To this end, the main technical ingredient of the algorithm lies in the $\ProfSwitch$ procedure, explained next. 

\paragraph{Policy Improvement.} 
The $\ProfSwitch$ procedure (Algorithm~\ref{alg:switch}) takes as input a positional agent policy $\agentpol$, computes its value vector $\val^{\agentpol}$ (using the method of Section~\ref{sec:rmc}), and then tries to improve $\agentpol$ in one of the following two ways:
\begin{compactitem}
\item \emph{Quantitative} improvement: A quantitative improvement is possible when for some state $s$ an action $a$ exists such that its worst-case one-step value $\nu(s,a)=\min_{p\in\uncert(s,a)}p\cdot\val^{\agentpol}$ strictly exceeds the current value $\val^{\agentpol}_s$. In this case, the agent switches to such an action $a$ for every state $s$ where an improvement is possible and returns the new policy (Lines \ref{alg:switch:begin:quant}--\ref{alg:switch:end:quant} of Algorithm~\ref{alg:switch}). The following lemma shows that the resulting policy is strictly better than $\agentpol$:
\begin{restatable}{lemma}{valuelemma}
\label{lem:value}
Let $\agentpol'=\ProfSwitch(\rmdp,\agentpol)$ be obtained by a quantitative
improvement. Then $\val^{\agentpol'}_s\ge\val^{\agentpol}_s$ for all $s\in\states$, and
$\val^{\agentpol'}_s>\val^{\agentpol}_s$ for some $s\in \states$.
\end{restatable}
\item \emph{Qualitative} improvement: 
If a quantitative improvement is not available, the algorithm checks whether a qualitative improvement is possible. To this end, the algorithm iterates over the value classes of $\val^{\agentpol}$:
\paragraph{Value Classes, Tight Actions, Tight Faces.} 
For $r\in[0,1]$, the \emph{value class}
$\vc{\agentpol}{r}=\{s\in\states\mid\val^{\agentpol}_s=r\}$ is the set of
states of value $r$ under $\agentpol$. There are at most $|\states|$ non-empty value classes.
An action $a\in\actions$ is \emph{tight} at $s$
(with respect to $\val^{\agentpol}$) if $\nu(s,a)=\val^{\agentpol}_s$; we
write $\tightacts{\val^{\agentpol}}{s}$ for the set of tight actions at
$s$, and call $\tightset{\val^{\agentpol}}{s}{a}$ the \emph{tight face}
at $(s,a)$ when $a$ is tight at $s$. If $a$ is tight at $s$ then every
$p\in\uncert(s,a)$ satisfies
$p\cdot\val^{\agentpol}\ge\val^{\agentpol}_s$; if $a$ is not tight at
$s$ and no action improves at $s$, then
$\min_p p\cdot\val^{\agentpol}<\val^{\agentpol}_s$.

For every value class $U_r = \vc{\agentpol}{r}$, the algorithm will then construct an RMDP $\rmdp_r$ together with a coloring $\coloring_r$ defined as follows: (i) the states of $\rmdp_r$ are those of $U_r$ together with a fresh terminal state $\bot$, (ii) the agent's actions are its tight actions, (iii) the uncertainty sets of $\rmdp_r$ are the tight faces of $\rmdp$ inside $U_r$ and the singleton $\{\delta_\bot\}$ (dirac distribution over $\bot$) outside $U_r$, and (iv) the coloring $\coloring_r$ is the same as $\coloring$ inside $U_r$ and assigns a fresh dominating odd color to $\bot$. 

The algorithm then calls the almost-sure parity procedure of~\cite{AsadiCGKS26} on $(\rmdp_r,\parity(\coloring_r))$ to compute the set $W_r \subseteq U_r$ as the set of state where the agent can guarantee $\parity(\coloring_r)$'s satisfaction with probability 1 together with its corresponding optimal policy $\agentpol_r$ (lines~\ref{alg:switch:begin:qual}--\ref{alg:switch:end:qual}). If $W_r\neq\emptyset$, then the agent can improve its policy by switching to $\agentpol_r$ on $W_r$. The following lemma shows that the resulting policy is strictly better than $\agentpol$:

\begin{restatable}{lemma}{quallemma}
\label{lem:qual}
Let $\agentpol'=\ProfSwitch(\rmdp,\agentpol)$ be obtained by a
qualitative improvement, i.e. a quantitative improvement is not possible and $W:=\bigcup_{r<1}W_r\neq\emptyset$. Then
$\val^{\agentpol'}_s\ge\val^{\agentpol}_s$ for all $s\in\states$, and
$\val^{\agentpol'}_s>\val^{\agentpol}_s$ for all $s\in W$.
\end{restatable}

Intuitively, Lemma~\ref{lem:qual} follows by the fact that if the environment tries to keep the run inside $U_r$, then the agent can guarantee $\parity(\coloring_r)$ with probability 1, and if the environment tries to escape $U_r$, then the agent can guarantee a value strictly larger than $r$ by the definition of tight faces. The full proof is given in Appendix~\ref{app:rmdp}.
\end{compactitem}

\begin{algorithm}[t]
\caption{$\StratImp(\rmdp)$}
\label{alg:si}
\begin{algorithmic}[1]
\REQUIRE Linearly Defined RMDP $\rmdp$, coloring $\coloring$
\ENSURE optimal agent policy $\agentpol^*$ and value vector $\val^{\rmdp}(\spec)$
\STATE pick an arbitrary positional policy $\agentpol$
\REPEAT
  \STATE $\agentpol_{\mathrm{old}} \gets \agentpol$
  \STATE $\agentpol \gets \ProfSwitch(\rmdp,\agentpol_{\mathrm{old}})$
\UNTIL{$\agentpol = \agentpol_{\mathrm{old}}$}
\STATE \textbf{return} $\agentpol$ and $\val^{\agentpol}$
\end{algorithmic}
\end{algorithm}

\begin{algorithm}[t]
\caption{$\ProfSwitch(\rmdp, \agentpol)$}
\label{alg:switch}
\begin{algorithmic}[1]
\REQUIRE RMDP $\rmdp$, coloring $\coloring$, agent
         policy $\agentpol$
\ENSURE policy $\agentpol'$ with $\agentpol' = \agentpol$ or
        $\val^{\agentpol'} \gneq \val^{\agentpol}$
\STATE $\val^{\agentpol} \gets \textsc{SolveRMC}(\rmdp^{\agentpol},
        \parity(\coloring))$
        \quad{\footnotesize$\triangleright$ Section~\ref{sec:rmc}} \label{alg:switch:solvermc}
\FOR{each $(s,a) \in \states\times\actions$} \label{alg:switch:begin:quant}
  \STATE $\nu(s,a) \gets \min_{p\in\uncert(s,a)} p\cdot\val^{\agentpol}$
         \quad{\footnotesize$\triangleright$ one linear program}
\ENDFOR
\STATE $I \gets \{\, s\in\states \mid \max_{a\in\actions}\nu(s,a)
        > \val^{\agentpol}_s \,\}$
\IF{$I \neq \emptyset$}
  \STATE $\agentpol'(s) \gets
          \begin{cases}
            \arg\max_{a\in\actions}\nu(s,a) & \text{if } s \in I,\\
            \agentpol(s) & \text{otherwise.}
          \end{cases}$
  \STATE \textbf{return} $\agentpol'$
\ENDIF \label{alg:switch:end:quant}
\STATE $W \gets \emptyset$ \label{alg:switch:begin:qual}
\FOR{each value $r < 1$ of $\val^{\agentpol}$ with
     $U_r = \vc{\agentpol}{r} \neq \emptyset$}
  \STATE Construct $\rmdp_r$ and coloring $\coloring_r$
         from $\nu$
  \STATE $(W_r, \agentpol_r) \gets \textsc{AlmostSureParity}(\rmdp_r,
          \parity(\coloring_r))$
   \STATE $W \gets W\cup W_r$
\ENDFOR \label{alg:switch:end:qual}
\IF{$W \neq \emptyset$}
  \STATE $\agentpol'(s) \gets
          \begin{cases}
            \agentpol_r(s) & \text{if } s \in W_r \text{ for some } r,\\
            \agentpol(s) & \text{otherwise.}
          \end{cases}$
  \STATE \textbf{return} $\agentpol'$
\ENDIF
\STATE \textbf{return} $\agentpol$
\end{algorithmic}
\end{algorithm}

\paragraph{Correctness and Complexity.} 
It follows directly from Lemma~\ref{lem:value} and Lemma~\ref{lem:qual} that every call to $\ProfSwitch$ either returns a strictly better policy or the same policy, in which case a fixpoint is reached. Hence, the policy is indeed improved in every iteration of $\StratImp$. Since there are only finitely many positional policies, the algorithm must terminate. The following theorem establishes the correctness and complexity of the algorithm:

\begin{restatable}{theorem}{totalcor}
\label{cor:total}
Given a linearly defined RMDP $\rmdp$ and a parity objective $\spec$:
\begin{compactitem}
\item $\StratImp(\rmdp)$ terminates after at most $|\actions|^{|\states|}$ calls to $\ProfSwitch$.
\item $\StratImp(\rmdp)$ returns an optimal positional agent policy $\agentpol^*$ together with the value vector $\val^{\rmdp}(\spec)$.
\item $\StratImp(\rmdp)$ runs in polynomial space.
\end{compactitem}
\end{restatable}

The polynomial space requirement is specifically important because the reduction in the proof of Proposition~\ref{prop:polytopic} explodes the state-space exponentially.

\section{Experimental Results}
\begin{figure*}[ht]
  \centering
  \includegraphics[width=0.8\textwidth]{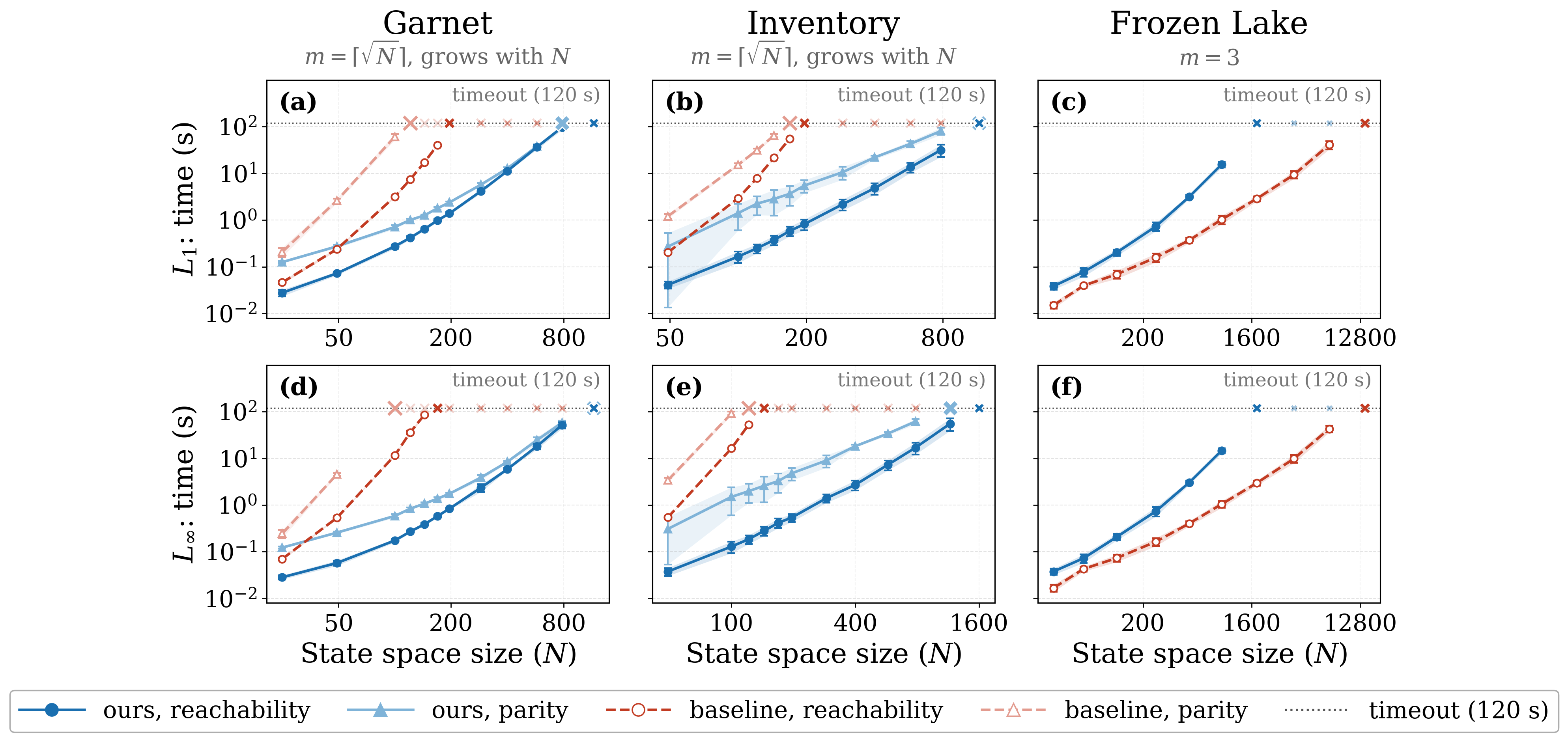}
  \caption{%
    Solve time as a function of the number of states~$N$, with one
    benchmark per column and one uncertainty norm per row: $L_1$
    (a--c) and $L_\infty$ (d--f). Blue solid lines show our algorithm and
    red dashed lines the baseline; within each color, dark circles mark
    the reachability objective and light triangles the parity objective
    (Frozen Lake is reachability-only). Each point is the mean over $10$
    random seeds, with error bars and shaded bands giving one standard
    deviation. The dotted horizontal line is the 120s
    timeout, and a~$\times$ marks a timed-out.
  }
  \label{fig:scaling}
\end{figure*}

We implemented our proposed method and conducted a series of experiments to evaluate its performance. The experiments were designed to assess the effectiveness and applicability of our approach for quantitative analysis of both general parity objectives and the special class of reachability objectives. We compare our method against the technique that reduces the RMDP into a stochastic game (as in Appendix~\ref{app:polytopic}) and solves it using the policy iteration algorithm of~\cite{chatterjee2005algorithms}. 

The experiments are designed to answer the following research questions:
\begin{compactitem}
    \item \textbf{RQ1:} How does our method perform in terms of computational efficiency compared to the baseline approach?
    \item \textbf{RQ2:} How does our method scale with increasing problem size and complexity?
    \item \textbf{RQ3:} What are the limitations of our method, and in what scenarios does it outperform the baseline approach?
\end{compactitem}

\paragraph{Experimental Setup.} We implemented our approach and the baseline method in Python, and used them to compute the quantitative value of the benchmarks explained next. All experiments were conducted on a machine with an Intel Core Ultra 5 225U processor and 16GB of RAM. We used a timeout of 120 seconds for each experiment.

\paragraph{Benchmarks.} We conducted experiments on three classes of benchmarks, where the branching factor $m$ denotes the number of possible successors of a state under a fixed action. Garnet and Inventory Management support both objectives, while Frozen Lake supports only reachability. In all cases, the uncertainty sets are $L_\infty$ or $L_1$ balls around the nominal distributions. The benchmarks are summarized as follows (full details are given in Appendix~\ref{app:experiments}):
\begin{compactitem}
    \item \textbf{Garnet} (adapted from~\cite{archibald1995generation}) is a family of randomly generated MDPs. The state-space consists of $N$ states partitioned into $k$ subsets, $G_0,\dots,G_k$. Every state has three available actions, each leading to $m=\lceil\sqrt{N}\rceil$ possible successors. Taking an action from a state in $G_0$ will either lead to a state in $G_0$ or to a state in one of the other subsets, while taking an action from a state in $G_i$ for $i>0$ will lead to a state in $G_i$. The transition probabilities are drawn uniformly at random from the probability simplex. As a reachability objective, each one of the subsets is considered to have a target and an trap state; the goal is to reach the target state of some subset. As a parity objective, each state is assigned a random color and the goal is to satisfy the respective parity condition. 
    
    \item \textbf{Inventory Management} (adapted from~\cite{iyengar2005robust}) models a warehouse with capacity $N$ whose state is the stock level. At each step the agent orders new units against an uncertain demand. Specifically, the agent will choose to buy $a \in \{0,\dots,a_{\max}\}$ new units and the demand $d$ is drawn uniformly from $\{0,\dots,d_{\max}\}$. The state will then transition from $s$ to $s+a-d$; if the demand exceeds the available stock ($d>s+a$), or the delivery pushes the stock over the capacity ($s+a-d>N$), the run instead ends in one of two absorbing failure states.
    We fix $d_{\max} = \lceil\sqrt{N}\rceil$ and $a_{\max} = \lceil0.75\,d_{\max}\rceil$. Once the stock reaches a threshold $T$, a larger order option becomes available. The reachability objective is to reach a stock level above a threshold, while the parity objective is to keep the stock level above a threshold infinitely often.

    \item \textbf{Frozen Lake} \cite{brockman2016gym} is the classic grid world where the agent walks on a $k\times k$ grid from the top-left cell toward the bottom-right. Each move proceeds in the intended direction or in one of the two perpendicular ones with probability $\tfrac13$ each, and a random $20\%$ of the cells are holes, which are absorbing failures. 
\end{compactitem}

\paragraph{Results.}
Figure \ref{fig:scaling} reports the solve times of both algorithms against
the state-space size, on a log--log scale; every configuration is
grown until each algorithm, in turn, exceeds the $120$-second timeout,
marked by a~$\times$. The returned values of both approaches were checked to match on all benchmarks. The picture cleanly separates two regimes,
governed by the branching factor $m$.

On \textbf{Garnet} and \textbf{Inventory Management}, where $m$ grows with the instance size, our algorithm scales substantially better than the baseline, across both norms and both objectives (panels a, b, d, e). The baseline's cost is driven by the number of extreme points of the uncertainty set, which grows rapidly with $m$: it already exceeds the timeout at a few hundred states, whereas our algorithm keeps solving an order of magnitude larger instances (up to $n\approx1200$ on Inventory). At the largest sizes where the baseline still completes, our algorithm is faster by roughly one to two orders of magnitude, and this gap widens as the instances grow. Consistent with the extra work of the parity decomposition, parity (light curves) is somewhat more expensive than reachability (dark curves) for both algorithms, so the baseline times out slightly earlier on parity than on reachability.

\textbf{Frozen Lake} is the opposite regime (panels c, f). Here the branching factor is fixed at $m\le3$, so each uncertainty set has only a handful of extreme points and the reduced stochastic game stays close in size to the original RMDP. The baseline therefore has the advantage: it is roughly an order of magnitude faster than our algorithm at matched sizes and scales to far larger grids (beyond $n\approx10^4$), while the overhead of our linear-programming inner solver dominates.

These results answer our research questions directly. Our method is the more efficient choice (RQ1) and scales to markedly larger instances (RQ2) precisely when the branching factor $m$ is large, since its per-iteration cost is polynomial in the size of the RMDP rather than in the number of extreme points of the uncertainty sets. Its limitation (RQ3) shows up when $m$ is small: when the uncertainty sets have few extreme points, the reduction to a stochastic game produces a compact game that the baseline solves faster.

Finally, we note that the number of iterations both in our approach and the baseline is observed to be smaller than the theoretical worst-case bound, which is exponential in the number of states. This is consistent with the known behavior of policy iteration on stochastic games, where the number of iterations is often observed to be polynomial in practice, despite the worst-case exponential bound. This suggests that the practical performance of our algorithms may be significantly better than the theoretical worst-case analysis would suggest (See Appendix~\ref{app:experiments} for details). 

\section{Conclusion and Future Works}

In this paper, we study the quantitative parity problem for $(s,a)$-rectangular
RMDPs with linearly defined uncertainty sets. We give a polynomial-time
algorithm for robust Markov chains, casting robust safety as a linear program
over occupation measures and reducing parity to safety through a
maximal-end-component analysis, and use it as a subroutine in a policy-iteration
algorithm for RMDPs that combines quantitative one-step improvements with
qualitative almost-sure improvements, runs in polynomial space, and computes the
exact values together with optimal positional policies for both the agent and
the environment. The decision problem for linearly defined RMDPs is in
$\mathit{NP}\cap\mathit{coNP}$ and as hard as turn-based stochastic parity games.
An interesting future question is whether the exact value is computable in
polynomial time. Other promising directions are to extend beyond linearly
defined sets to nonlinear convex uncertainty, such as ellipsoidal or
divergence-based balls, to relax $(s,a)$-rectangularity, and to study the
learning setting in which the uncertainty sets are estimated from samples.

\bibliography{bibliography}
\clearpage
\appendix
\setcounter{secnumdepth}{2} 
\onecolumn

\section{Related Work}
\label{app:related}

\paragraph{Robust MDPs.}
RMDPs were introduced for sequential decision making under uncertain
transition probabilities, with the foundational works
of~\cite{iyengar2005robust,nilim2005robust} establishing robust dynamic
programming for rectangular uncertainty. We refer to~\citet{suilen2024robust}
for a recent survey. Subsequent works studied general robust policy-iteration
methods~\cite{wiesemann2013,DBLP:journals/informs/KaufmanS13}, as well as faster
algorithms for specific uncertainty classes such as $L_1$ and
$L_\infty$~\cite{ho2021partial,behzadian2021fast}. Recent works study exact
computation and the complexity of policy iteration for discounted
RMDPs~\cite{asadi2026strongly}, while others consider the average-reward and
bounded-parameter settings~\cite{wang2023robust,chatterjee2024solving,tewari2007bounded,DBLP:journals/ai/GivanLD00}.
All of these results consider quantitative objectives such as discounted-sum,
average reward, and finite horizon, whereas our focus is logical objectives.

\paragraph{Parity Objectives and Strategy Improvement.}
Parity objectives and their quantitative analysis have been studied extensively
for MDPs and two-player stochastic games. We refer
to~\citet{model-checking-book,chatterjee2012survey} for an overview, and
to~\citet{zielonka1998infinite,parys2019parity} for algorithms solving parity
games. Our policy-iteration algorithm is, in spirit, close to the
strategy-improvement paradigm for stochastic parity and Rabin
games~\cite{chatterjee2005algorithms,chatterjee2006strategy}, which likewise
iterates over positional strategies and alternates value improvements with
qualitative sub-computations. The essential difference is that in our robust
setting the environment ranges over a continuum of distributions in an
uncertainty polytope rather than over the finitely many vertices of a stochastic
game. Our improvement therefore works directly with the linearly defined
uncertainty sets, through the value classes, tight actions, and tight faces of
Section~\ref{sec:rmdp}, and never constructs the exponentially larger game.

\paragraph{Logical Objectives on Uncertain MDPs.}
The closest prior work considers temporal-logic control of uncertain MDPs under a
constant-support assumption and approximates satisfaction probabilities using
value iteration~\cite{wolff2012robust}. \citet{AsadiCGKS26} remove structural
assumptions and solve the positive and almost-sure reachability and parity
problems using uncertainty-set oracles. Our qualitative improvement step invokes
their almost-sure parity procedure as a sub-routine, whereas the exact
quantitative problem we solve here remained open.

\paragraph{Reinforcement Learning.}
RMDPs are the model underlying robust reinforcement learning, where the goal is
to learn policies that hedge against misspecification of the transition model.
Adversarial and model-uncertainty formulations optimize the worst-case return
over an uncertainty set~\cite{pinto2017robust,wang2021online}, and we refer
to~\citet{moos2022robust} for a survey. This line targets quantitative reward
objectives, whereas we compute exact values of $\omega$-regular objectives.
Orthogonally, a body of work learns policies for $\omega$-regular and
temporal-logic objectives in the non-robust setting, typically by reducing the
objective to a reachability reward on a product with an
automaton~\cite{hahn2019omega,bozkurt2020control}. Our setting combines the two
axes: exact quantitative analysis of $\omega$-regular objectives against
adversarial transition uncertainty.

\section{Linear Descriptions of $L_\infty$ and $L_1$ Uncertainty Sets}
\label{app:uncertainty}

Recall the Assumption of Section~\ref{sec:rmc}: the uncertainty set of each state $s$ is a polytope $\uncert(s) = \{ p \in \Delta(\states) : A_s\, p \le b_s \}$. Here we show that the two most common uncertainty models, the $L_\infty$ and $L_1$ balls around a nominal distribution, fit this form (the $L_1$ case after introducing auxiliary variables). Fix a state $s$ with nominal distribution $\hat{p}_s \in \Delta(\states)$ and radius $\delta_s \ge 0$; write $\hat{p}_{s,s'}$ for the nominal probability of the successor $s'$.

\paragraph{$L_\infty$ ball.} The set
\[
    \uncert(s) = \big\{\, p \in \Delta(\states) : \| p - \hat{p}_s \|_\infty \le \delta_s \,\big\}
\]
is cut out, with no auxiliary variables, by the box constraints
\[
    -\delta_s \;\le\; p_{s'} - \hat{p}_{s,s'} \;\le\; \delta_s
    \qquad \text{for every } s' \in \states,
\]
i.e. $p_{s'} \le \hat{p}_{s,s'} + \delta_s$ and $-p_{s'} \le \delta_s - \hat{p}_{s,s'}$. Stacking these $2|\states|$ inequalities gives $A_s = \left[\begin{smallmatrix} I \\ -I \end{smallmatrix}\right]$ and $b_s = \left[\begin{smallmatrix} \hat{p}_s + \delta_s \mathbf{1} \\ -\hat{p}_s + \delta_s \mathbf{1} \end{smallmatrix}\right]$, so $\uncert(s) = \{ p \in \Delta(\states) : A_s\, p \le b_s \}$ exactly as in the Assumption.

\paragraph{$L_1$ ball.} The set
\[
    \uncert(s) = \big\{\, p \in \Delta(\states) : \| p - \hat{p}_s \|_1 \le \delta_s \,\big\}
    = \Big\{\, p \in \Delta(\states) : \textstyle\sum_{s' \in \states} | p_{s'} - \hat{p}_{s,s'} | \le \delta_s \,\Big\}
\]
has, in general, exponentially many facets, but it becomes linear after introducing one auxiliary variable $y_{s'}$ per successor to model $| p_{s'} - \hat{p}_{s,s'} |$:
\[
    y_{s'} \ge p_{s'} - \hat{p}_{s,s'},
    \qquad
    y_{s'} \ge \hat{p}_{s,s'} - p_{s'},
    \qquad (s' \in \states),
    \qquad
    \sum_{s' \in \states} y_{s'} \le \delta_s .
\]
At an optimum each $y_{s'}$ equals $| p_{s'} - \hat{p}_{s,s'} |$, so the projection onto the $p$-coordinates of the polytope defined by these $2|\states| + 1$ inequalities (over the $2|\states|$ variables $(p, y)$) is exactly the $L_1$ ball. Writing the constraints as $A_s \left[\begin{smallmatrix} p \\ y \end{smallmatrix}\right] \le b_s$ presents $\uncert(s)$ as the projection of a polytope, using only a linear number of auxiliary variables and inequalities.

\paragraph{Extending the LP to auxiliary variables.} The $L_1$ description uses auxiliary variables, so $\uncert(s)$ is a \emph{projected} polytope rather than the plain form $\{ p : A_s p \le b_s \}$. The program~$(\cross)$ and its correctness extend to this case without change. Indeed, the linearisation of Lemma~\ref{lem:app:lin} rewrites the admissibility constraint $p \in \uncert(s)$, with $p = \boldsymbol{f}_s / x_s$, by multiplying through by $x_s \ge 0$; applying the \emph{same} scaling to the auxiliary variables (replacing $y$ by $y' = x_s\, y$) turns $A_s \left[\begin{smallmatrix} p \\ y \end{smallmatrix}\right] \le b_s$ into the linear constraint $A_s \left[\begin{smallmatrix} \boldsymbol{f}_s \\ y' \end{smallmatrix}\right] \le x_s\, b_s$ over the flow and (scaled) auxiliary variables. Thus, for $L_\infty$, $L_1$, and general (possibly projected) polytopes alike, the value $\val^{\rmc}_{s}(\safe(\target))$ is computed by a single polynomial-size linear program, and Lemma~\ref{lem:rmc:reachability} applies uniformly.

\section{Proofs of Lemma~\ref{lem:rmc:reachability} and Theorem~\ref{thm:rmc:safety-poly}}
\label{app:rmc:reach}

Throughout this section we fix a linearly defined RMC $\rmc = (\states, \uncert)$ with, for every $s \in \states$, $\uncert(s) = \{ p \in \Delta(\states) : A_s\, p \le b_s \}$, a safe set $\target \subseteq \states$, and an initial state $\sinit$. A \emph{positional} environment policy is a map $\envpol \colon \states \to \Delta(\states)$ with $\envpol(s) \in \uncert(s)$ for all $s$; we write $p^{\envpol}_{s,s'} = \envpol(s)(s')$ for its one-step probabilities. Since an RMC has a single action, the safety value from Section~\ref{sec:RMC:reachability} is
\[
    \val^{\rmc}_{s}(\safe(\target)) = \inf_{\envpol} \Pr\nolimits^{\envpol}_{s}[\safe(\target)],
\]
the infimum ranging over all environment policies. The preprocessing computes the sets
\[
    \states^{=1} = \{\, s \in \states : \val^{\rmc}_{s}(\safe(\target)) = 1 \,\}, \quad
    \states^{=0} = \{\, s \in \states : \val^{\rmc}_{s}(\safe(\target)) = 0 \,\}, \quad
    \states^{?} = \states \setminus (\states^{=0} \cup \states^{=1}),
\]
and makes the states of $\states^{=0} \cup \states^{=1}$ absorbing sinks. Every unsafe state (outside $\target$) has safety value $0$, so $\states \setminus \target \subseteq \states^{=0}$ and hence $\states^{?} \subseteq \target$. We prove that for $\sinit \in \states^{?}$ the optimum of $(\cross)$ equals $\val^{\rmc}_{\sinit}(\safe(\target))$; for $\sinit \in \states^{=1}$ (resp. $\states^{=0}$) the value is $1$ (resp. $0$) by definition. 

A \emph{reachability} objective $\reach(T)$ specifies the set of all runs that reach a state in $T$ eventually, i.e.
\[
\reach(T) = \{\pi \in \states^\omega | \exists i: \pi_i \in T\}
\]

\subsection{Occupation Measures}

\begin{definition}[Expected visits and flows]\label{def:app:visits}
Fix a positional environment policy $\envpol$. For $s, s' \in \states$, the \emph{expected number of visits} to $s$ and the \emph{expected flow} along $(s,s')$, both from $\sinit$, are
\[
    x_s^{\envpol} = \sum_{t=0}^{\infty} \Pr\nolimits^{\envpol}_{\sinit}[s_t = s], \qquad
    f_{s,s'}^{\envpol} = \sum_{t=0}^{\infty} \Pr\nolimits^{\envpol}_{\sinit}[s_t = s,\, s_{t+1} = s'].
\]
\end{definition}

By the Markov property of a positional $\envpol$, $\Pr^{\envpol}_{\sinit}[s_t = s,\, s_{t+1} = s'] = p^{\envpol}_{s,s'}\, \Pr^{\envpol}_{\sinit}[s_t = s]$; summing over $t$ gives the factorisation
\begin{equation}\label{eq:app:factor}
    f_{s,s'}^{\envpol} = x_s^{\envpol}\, p_{s,s'}^{\envpol},
    \qquad\text{hence}\qquad
    \sum_{s' \in \states} f_{s,s'}^{\envpol} = x_s^{\envpol}
    \quad \text{whenever } x_s^{\envpol} < \infty.
\end{equation}
Whenever $x_s^\envpol < \infty$ for all $s \in \states^?$, conditioning on the state at time $t-1$ and summing gives $\mathsf{Cons}$, and $\mathsf{Split}$ is the second identity in~\eqref{eq:app:factor}; together with $\mathsf{NNeg}$ (immediate) the pair $(x^\envpol, f^\envpol)$ then satisfies the constraints of $(\cross)$ by Lemma~\ref{lem:app:lin}. Moreover, since $\states^{=1}$ is absorbing, a play crosses from $\states^{?}$ into $\states^{=1}$ at most once and exactly when it reaches $\states^{=1}$, so the objective of $(\cross)$ evaluated at this occupation is
\begin{equation}\label{eq:app:obj}
    \sum_{s \in \states^{?},\, w \in \states^{=1}} f_{s,w}^{\envpol}
    = \Pr\nolimits^{\envpol}_{\sinit}[\reach(\states^{=1})].
\end{equation}

\subsection{Correctness of the Preprocessing}

\begin{lemma}[Qualitative preprocessing]\label{lem:app:qual}
The sets $\states^{=1}$ and $\states^{=0}$ are computable in polynomial time, and making states in $\states^{=0}$ unsafe absorbing and states in $\states^{=1}$ safe absorbing does not change the safety value of any state in $\states^?$.
\end{lemma}

\begin{proof}
The qualitative sets $\states^{=1}$ and $\states^{=0}$ are computable in polynomial time by \citep{AsadiCGKS26}. From a state of $\states^{=1}$ the play is safe under \emph{every} policy, and from a state of $\states^{=0}$ some policy is unsafe almost surely; making these states absorbing therefore only relabels play that has already committed to value $1$ or $0$, and does not change any safety value.
\end{proof}



\begin{lemma}\label{lem:app:preproc}
Let $\envpol^{\star}$ be a positional environment policy attaining $\val^{\rmc}_{\sinit}(\safe(\target))$. Then from $\sinit \in \states^{?}$ the play reaches $\states^{=0} \cup \states^{=1}$ almost surely; consequently $x_s^{\envpol^{\star}} < \infty$ for every $s \in \states^{?}$, and under $\envpol^{\star}$ a play is safe if and only if it reaches $\states^{=1}$, so that $\Pr^{\envpol^{\star}}_{\sinit}[\safe(\target)] = \Pr^{\envpol^{\star}}_{\sinit}[\reach(\states^{=1})]$.
\end{lemma}

\begin{proof}
Suppose the induced Markov chain has a BSCC $B$ reachable from $\sinit$ with $B \subseteq \states^{?}$. Since $\states^{?} \subseteq \target$ and $B$ is a BSCC, the play from any $u \in B$ stays in $B \subseteq \target$ forever and is therefore safe with probability $1$. Fix such a $u$. As $u \in \states^{?}$, we have $\val^{\rmc}_{u}(\safe(\target)) < 1$, so some policy $\envpol'$ is unsafe from $u$ with positive probability. Consider the policy that follows $\envpol^{\star}$ until $u$ is first visited and then switches to $\envpol'$. Since $\states^{=0} \cup \states^{=1}$ is absorbing, $u$ is reached from $\sinit$ with some probability $q > 0$ along a path inside $\states^{?} \subseteq \target$, and under $\envpol^{\star}$ these are exactly the safe plays that visit $u$; replacing their safe continuation by $\envpol'$ hence lowers the safety probability from $\sinit$ strictly below $\Pr^{\envpol^{\star}}_{\sinit}[\safe(\target)]$, contradicting the optimality of $\envpol^{\star}$ at $\sinit$. Hence every BSCC reachable from $\sinit$ meets $\states^{=0} \cup \states^{=1}$, so the play reaches $\states^{=0} \cup \states^{=1}$ almost surely, the hitting time $\tau$ has finite expectation, and $x_s^{\envpol^{\star}} \le \mathbb{E}^{\envpol^{\star}}_{\sinit}[\tau] < \infty$ for every $s \in \states^{?}$. Finally, once absorbed in $\states^{=1}$ the play is safe, whereas from $\states^{=0}$ the optimal $\envpol^{\star}$ is unsafe almost surely; hence, under $\envpol^{\star}$, being safe coincides with reaching $\states^{=1}$, giving $\Pr^{\envpol^{\star}}_{\sinit}[\safe(\target)] = \Pr^{\envpol^{\star}}_{\sinit}[\reach(\states^{=1})]$.
\end{proof}

\subsection{Linearization of the Uncertainty}

\begin{lemma}\label{lem:app:lin}
Fix $s \in \states^{?}$, a scalar $x_s \ge 0$, and a vector $\boldsymbol{f}_s = (f_{s,s'})_{s' \in \states}$. The constraints
$\boldsymbol{f}_s \ge 0$, $\sum_{s' \in \states} f_{s,s'} = x_s$, and $A_s\, \boldsymbol{f}_s \le x_s\, b_s$ hold if and only if there exists $p \in \uncert(s)$ with $\boldsymbol{f}_s = x_s\, p$.
\end{lemma}

\begin{proof}
($\Leftarrow$) Given $p \in \uncert(s)$ with $\boldsymbol{f}_s = x_s\, p$: non-negativity and $\sum_{s'} f_{s,s'} = x_s \sum_{s'} p_{s'} = x_s$ follow from $p \in \Delta(\states)$, and multiplying $A_s\, p \le b_s$ by $x_s \ge 0$ gives $A_s\, \boldsymbol{f}_s \le x_s\, b_s$.

($\Rightarrow$) If $x_s > 0$, set $p = \boldsymbol{f}_s / x_s$. Then $p \ge 0$, $\sum_{s'} p_{s'} = 1$, so $p \in \Delta(\states)$; dividing $A_s\, \boldsymbol{f}_s \le x_s\, b_s$ by $x_s$ gives $A_s\, p \le b_s$, hence $p \in \uncert(s)$ and $\boldsymbol{f}_s = x_s\, p$. If $x_s = 0$, then $\boldsymbol{f}_s \ge 0$ and $\sum_{s'} f_{s,s'} = 0$ force $\boldsymbol{f}_s = 0 = x_s\, p$ for any $p \in \uncert(s)$, and $\uncert(s) \ne \emptyset$.
\end{proof}

In particular, $\mathsf{Split} \wedge \mathsf{NNeg} \wedge \mathsf{Adm}$ at a state $s \in \states^{?}$ holds exactly when the outgoing flow is $\boldsymbol{f}_s = x_s\, p$ for some admissible distribution $p \in \uncert(s)$.

\subsection{Proof of Lemma~\ref{lem:rmc:reachability}}

\begin{proof}[Proof of Lemma~\ref{lem:rmc:reachability}]
Fix $\sinit \in \states^{?}$. We show that the optimum of $(\cross)$ equals $\val^{\rmc}_{\sinit}(\safe(\target))$, via two inequalities.

\emph{The optimum is at most the value.}
Let $\envpol^{\star}$ be an optimal positional policy; by Lemma~\ref{lem:app:preproc} it has finite occupation. By~\eqref{eq:app:factor} and Lemma~\ref{lem:app:lin}, the occupation $(x^{\envpol^{\star}}, f^{\envpol^{\star}})$ satisfies $\mathsf{NNeg} \wedge \mathsf{Cons} \wedge \mathsf{Split} \wedge \mathsf{Adm}$, so it is feasible for $(\cross)$. By~\eqref{eq:app:obj} and Lemma~\ref{lem:app:preproc}, its objective is $\Pr^{\envpol^{\star}}_{\sinit}[\reach(\states^{=1})] = \Pr^{\envpol^{\star}}_{\sinit}[\safe(\target)] = \val^{\rmc}_{\sinit}(\safe(\target))$. Since $(\cross)$ minimises, its optimum is at most the value.

\emph{The value is at most the optimum.}
Let $(x^{\star}, f^{\star})$ be any feasible solution of $(\cross)$. Define a positional environment policy $\envpol^{\circ}$ by
\[
    \envpol^{\circ}(s) =
    \begin{cases}
        \boldsymbol{f}^{\star}_s / x^{\star}_s & \text{if } s \in \states^{?} \text{ and } x^{\star}_s > 0,\\
        p^{\dagger} \in \uncert(s) \text{ (arbitrary)} & \text{otherwise.}
    \end{cases}
\]
By Lemma~\ref{lem:app:lin}, $\envpol^{\circ}(s) \in \uncert(s)$, so $\envpol^{\circ}$ is a valid positional policy, and $x^{\star}_s\, p^{\envpol^{\circ}}_{s,s'} = f^{\star}_{s,s'}$ for all $s \in \states^{?}$ (both sides vanish when $x^{\star}_s = 0$, as then $\boldsymbol{f}^{\star}_s = 0$).

We show $x^{\envpol^{\circ}}_s \le x^{\star}_s$ for every $s \in \states^{?}$ by proving $x^{\envpol^{\circ}, N}_s \le x^{\star}_s$ for the truncated occupation $x^{\envpol^{\circ}, N}_s = \sum_{t=0}^{N} \Pr^{\envpol^{\circ}}_{\sinit}[s_t = s]$, by induction on $N$. For $N = 0$, $x^{\envpol^{\circ},0}_s = \mathbf{1}[s=\sinit] \le x^{\star}_s$ by $\mathsf{Cons}$ and $f^{\star} \ge 0$; and the inductive step uses $x^{\star}_{s'} p^{\envpol^{\circ}}_{s',s} = f^{\star}_{s',s}$ and $\mathsf{Cons}$,
\[
    x^{\envpol^{\circ},N}_s = \mathbf{1}[s=\sinit] + \sum_{s' \in \states^{?}} x^{\envpol^{\circ},N-1}_{s'}\, p^{\envpol^{\circ}}_{s',s}
    \le \mathbf{1}[s=\sinit] + \sum_{s' \in \states^{?}} f^{\star}_{s',s} = x^{\star}_s.
\]
Letting $N \to \infty$ gives $x^{\envpol^{\circ}}_s \le x^{\star}_s < \infty$; in particular $\envpol^{\circ}$ has finite occupation, so it reaches $\states^{=0} \cup \states^{=1}$ almost surely from $\sinit$. Using~\eqref{eq:app:obj}, $x^{\envpol^{\circ}}_s \le x^{\star}_s$, and $x^{\star}_s p^{\envpol^{\circ}}_{s,w} = f^{\star}_{s,w}$,
\begin{equation}\label{eq:app:objbound}
    \Pr\nolimits^{\envpol^{\circ}}_{\sinit}[\reach(\states^{=1})]
    = \sum_{\substack{s \in \states^{?} \\ w \in \states^{=1}}} x^{\envpol^{\circ}}_s\, p^{\envpol^{\circ}}_{s,w}
    \le \sum_{\substack{s \in \states^{?} \\ w \in \states^{=1}}} x^{\star}_s\, p^{\envpol^{\circ}}_{s,w}
    = \sum_{\substack{s \in \states^{?} \\ w \in \states^{=1}}} f^{\star}_{s,w}.
\end{equation}
It remains to bound the safety value by $\Pr^{\envpol^{\circ}}_{\sinit}[\reach(\states^{=1})]$. Consider the policy $\widehat{\envpol}$ that follows $\envpol^{\circ}$ until the play first reaches $\states^{=0} \cup \states^{=1}$ (which happens almost surely), then, from a state of $\states^{=1}$, plays to stay safe (possible, since safety value $1$ means safe under every policy), and from a state of $\states^{=0}$, plays an unsafe-almost-surely continuation (possible, since safety value $0$). Then a $\widehat{\envpol}$-play is safe exactly when it is absorbed in $\states^{=1}$, so
\[
    \val^{\rmc}_{\sinit}(\safe(\target)) \;\le\; \Pr\nolimits^{\widehat{\envpol}}_{\sinit}[\safe(\target)] \;=\; \Pr\nolimits^{\envpol^{\circ}}_{\sinit}[\reach(\states^{=1})]
    \;\le\; \sum_{\substack{s \in \states^{?} \\ w \in \states^{=1}}} f^{\star}_{s,w},
\]
using~\eqref{eq:app:objbound}. As this holds for every feasible solution, the value is at most the optimum of $(\cross)$.

Combining the two inequalities, the optimum of $(\cross)$ equals $\val^{\rmc}_{\sinit}(\safe(\target))$.
\end{proof}
\subsection{Proof of Theorem~\ref{thm:rmc:safety-poly}}
\begin{proof}[Proof of Theorem~\ref{thm:rmc:safety-poly}]
The sets $\states^{=1}, \states^{=0}$ are computable in polynomial time \citep{AsadiCGKS26}. The program $(\cross)$ has $O(|\states|^2)$ variables and, by the Assumption, a number of constraints polynomial in the encoding of $\rmc$, and linear programs of polynomial size are solvable in polynomial time. States in $\states^{=1}$ and $\states^{=0}$ have known safety values $1$ and $0$; for the remaining states, solving $(\cross)$ once per source $\sinit \in \states^{?}$ yields all values $\val^{\rmc}_{s}(\safe(\target))$ and, corresponding optimal positional policies, all in polynomial time.
\end{proof}

\section{$\mecdecomp$ Procedure for Computing MEC Decompositions of RMCs} \label{app:rmc:mec-decomp}

Throughout, fix an RMC $\rmc = (\states, \uncert)$ with linearly defined uncertainty. We recall the notion of maximal end-component from Definition~\ref{def:ec} and give the $\mecdecomp$ procedure that computes them; its correctness is established in Theorem~\ref{thm:rmc:mec-decomp}.

\paragraph{Internal transitions.} For a set $B \subseteq \states$ and a state $s \in B$, let
\[
    \uncert_B(s) \;=\; \{\, p \in \uncert(s) \;:\; \supp(p) \subseteq B \,\}
\]
be the set of admissible distributions at $s$ whose support stays inside $B$, i.e. the one-step choices with which the environment can keep the play within $B$. We say $B$ is \emph{closed} if $\uncert_B(s) \ne \emptyset$ for every $s \in B$; equivalently, from every state of $B$ the environment has at least one distribution that does not leave $B$. The \emph{internal transition graph} of $B$ is the directed graph $G_B = (B, E_B)$ whose edges record the transitions the environment can take while remaining in $B$,
\[
    E_B \;=\; \{\, (s, s') \in B \times B \;:\; \exists\, p \in \uncert_B(s),\ p_{s'} > 0 \,\}.
\]

This is exactly the notation of Definition~\ref{def:ec}: the end-components of $\rmc$ are the non-empty closed sets $B$ whose internal transition graph $G_B$ is strongly connected, and a \emph{maximal end-component} (MEC) is an end-component that is not strictly contained in another.

\paragraph{Algorithm.} Algorithm~\ref{algo:mec-decomp} computes the MEC decomposition of $\rmc$. Following the standard MDP procedure, it repeatedly \emph{closes} a candidate set $B$ by discarding every state from which the environment cannot avoid leaving $B$, and then \emph{splits} the closed set into the strongly connected components of $G_B$, recursing on each proper component. Both operations reduce to polynomially many linear-programming queries against the description of $\uncert$.

\begin{algorithm}[ht]
\caption{$\mecdecomp(B)$: MEC decomposition of an RMC}
\label{algo:mec-decomp}
\begin{algorithmic}[1]
\REQUIRE A subset $B \subseteq \states$ of an RMC $\rmc$ with linearly defined uncertainty; the top-level call uses $B = \states$.
\ENSURE The set of MECs contained in $B$.
\REPEAT
    \STATE $\mathsf{Leak} \gets \{\, s \in B \;:\; \uncert_B(s) = \emptyset \,\}$
    \STATE $B \gets B \setminus \mathsf{Leak}$
\UNTIL{$\mathsf{Leak} = \emptyset$}
\IF{$B = \emptyset$}
    \STATE \textbf{return} $\emptyset$
\ENDIF
\STATE Compute the edge set $E_B$ and the SCCs $C_1, \dots, C_k$ of $G_B$
\IF{$k = 1$}
    \STATE \textbf{return} $\{B\}$ \COMMENT{$B$ is a MEC}
\ENDIF
\STATE $\mathsf{Result} \gets \emptyset$
\FOR{$i = 1, \dots, k$}
    \STATE $\mathsf{Result} \gets \mathsf{Result} \cup \mecdecomp(C_i)$
\ENDFOR
\STATE \textbf{return} $\mathsf{Result}$
\end{algorithmic}
\end{algorithm}

\paragraph{LP primitives.} The two operations on $\uncert$ are realised as linear programs:
\begin{itemize}
    \item \emph{Closure test $\uncert_B(s) \ne \emptyset$.} Add the equalities $p_{s'} = 0$ for every $s' \notin B$ to the linear description $A_s\, p \le b_s$ of $\uncert(s)$ and test feasibility over the simplex; the set $\uncert_B(s)$ is non-empty iff this program is feasible.
    \item \emph{Edges of $G_B$.} For each ordered pair $(s, s') \in B \times B$, maximise $p_{s'}$ subject to $p \in \uncert_B(s)$; the edge $(s,s')$ belongs to $E_B$ iff the optimum is strictly positive.
\end{itemize}
The strongly connected components of $G_B$ are then computed by Tarjan's algorithm in time linear in $|B| + |E_B|$.

\begin{lemma}\label{lem:rmc:mec-disjoint}
The maximal end-components of $\rmc$ are pairwise disjoint.
\end{lemma}

\begin{proof}
Let $M_1, M_2$ be MECs with a common state $s \in M_1 \cap M_2$; we show that $M_1 \cup M_2$ is an end-component, which, since two distinct MECs cannot contain one another, strictly contains $M_1$ and contradicts its maximality. For closure, take any $s' \in M_1 \cup M_2$, say $s' \in M_i$. As $M_i$ is closed, $\uncert_{M_i}(s') \ne \emptyset$; and since $M_i \subseteq M_1 \cup M_2$, a distribution supported in $M_i$ is also supported in $M_1 \cup M_2$, so $\uncert_{M_i}(s') \subseteq \uncert_{M_1 \cup M_2}(s')$ and the latter is non-empty. For strong connectivity of $G_{M_1 \cup M_2}$, note that relaxing the support constraint from $\subseteq M_i$ to $\subseteq M_1 \cup M_2$ only adds edges, so $G_{M_i}$ is a subgraph of $G_{M_1 \cup M_2}$ restricted to $M_i$; hence each $M_i$ is internally strongly connected in $G_{M_1 \cup M_2}$, and the shared state $s$ links the two parts, making $M_1 \cup M_2$ strongly connected. Thus $M_1 \cup M_2$ is an end-component.
\end{proof}

\begin{theorem}\label{thm:rmc:mec-decomp}
Algorithm~\ref{algo:mec-decomp} returns the set of maximal end-components of $\rmc$, using $O(|\states|^3)$ linear-programming queries each of size polynomial in the encoding of $\rmc$, and hence runs in polynomial time. Moreover, a \emph{single} invocation $\mecdecomp(B)$, excluding the work of its recursive calls, performs $O(|B|^2)$ linear-programming queries.
\end{theorem}

\begin{proof}
We first argue correctness and then bound the number of LP queries. Both parts rely on the invariant that at every call $\mecdecomp(B)$, every MEC of $\rmc$ that intersects $B$ is contained in $B$. The invariant holds for the top-level call with $B = \states$, since every MEC intersects $\states$ and is contained in it.

\emph{Closure preserves end-components.}
The closure loop removes a state $s$ from $B$ only when $\uncert_B(s) = \emptyset$. Let $E$ be any end-component with $E \subseteq B$ and $s \in E$. Since $E$ is closed, $\uncert_E(s) \ne \emptyset$, and $\uncert_E(s) \subseteq \uncert_B(s)$ because $E \subseteq B$; hence $\uncert_B(s) \ne \emptyset$ and $s$ is not removed. Therefore every end-component contained in $B$ survives the closure loop. Let $B^{\star}$ be the closed set obtained when the loop terminates.

\emph{Every end-component lies in a single SCC of $G_{B^{\star}}$.}
For an end-component $E \subseteq B^{\star}$, the support condition $\supp(p) \subseteq E$ is stronger than $\supp(p) \subseteq B^{\star}$, so $G_E$ is a subgraph of $G_{B^{\star}}$ restricted to $E$. Hence $E$ is strongly connected inside $G_{B^{\star}}$ and is contained in a single strongly connected component.

\emph{Base case: one SCC.}
Suppose $G_{B^{\star}}$ is strongly connected (a single SCC). Then $B^{\star}$ is non-empty, closed, and strongly connected, so it is an end-component. By Lemma~\ref{lem:rmc:mec-disjoint} it lies in a unique MEC $M$ of $\rmc$; since $M$ intersects $B$, the invariant gives $M \subseteq B$, and $M$ survives the closure loop, so $M \subseteq B^{\star}$. Combined with $B^{\star} \subseteq M$ (as $B^{\star}$ is an end-component contained in the MEC $M$), we get $B^{\star} = M$. Thus the algorithm correctly returns the single MEC $\{B^{\star}\}$.

\emph{Recursive case: several SCCs.}
If $G_{B^{\star}}$ has SCCs $C_1, \dots, C_k$ with $k \ge 2$, then by the previous paragraph every end-component inside $B^{\star}$ lies in some $C_i$, and the algorithm recurses on each $C_i$. The invariant is maintained for each recursive call: any MEC $M$ intersecting $C_i$ also intersects $B$, so $M \subseteq B$ by the parent invariant, $M \subseteq B^{\star}$ by closure, and $M \subseteq C_i$ since it lies in a single SCC. By induction on the recursion, the calls collectively return exactly the MECs of $\rmc$ contained in $B$; for the top-level call $B = \states$ this is the full MEC decomposition.

\emph{Complexity.}
A single call, excluding its recursive calls, performs $O(|B|^2)$ LP queries: the closure loop runs at most $|B|$ rounds, each testing the $|B|$ feasibility programs $\uncert_B(s) \ne \emptyset$, and building $G_B$ uses one maximisation LP per ordered pair in $B \times B$; the SCC computation is linear in $|B| + |E_B|$ and uses no LPs. This is the second claim of the statement.

We bound the total by counting the recursion level by level. There are at most $|\states|$ levels, because each child $C_i$ is a \emph{proper} subset of its parent: $C_i \subseteq B^{\star} \subseteq B$, and recursion happens only when $B^{\star}$ splits into $k \ge 2$ SCCs. The sets processed at any fixed level are pairwise disjoint, since the children of a node are the pairwise-disjoint SCCs of $B^{\star}$ and children of distinct nodes stay inside the disjoint sets of their parents. Hence each level costs at most $\sum_{B \text{ at that level}} O(|B|^2) \le O\big( \big( \sum_{B} |B| \big)^2 \big) = O(|\states|^2)$ LP queries, and the at most $|\states|$ levels cost $O(|\states|^3)$ in total. Each query is a linear program of size polynomial in the encoding of $\rmc$ and is solvable in polynomial time, so the whole algorithm runs in polynomial time.
\end{proof}

\section{Proofs of Lemmas~\ref{thm:rmc:almost-sure-parity} and~\ref{lem:rmc:parity}} \label{app:rmc:parity}

Throughout this section we fix a linearly defined RMC $\rmc = (\states, \uncert)$, a coloring $\coloring \colon \states \to [d]$, and an initial state $\sinit$. Recall from Section~\ref{sec:RMC:parity} that a play $\pi$ \emph{violates} $\parity(\coloring)$ exactly when its dominant color $\max\{\coloring(s) : s \in \Inf(\pi)\}$ is odd, and write
\[
    W_{\mathrm{odd}} = \asoddsat(\states)
\]
for the output of Algorithm~\ref{algo:as-odd-sat}; the internal transition graph $G_B$ and the restricted uncertainty sets $\uncert_B$ are those of Appendix~\ref{app:rmc:mec-decomp}. The proofs rest on two properties of $W_{\mathrm{odd}}$: the environment can confine a play to it while forcing an odd dominant color (Lemma~\ref{lem:app:parity:target}), and every violating play must reach it (Lemma~\ref{lem:app:parity:bound}).

\begin{lemma}[Confinement]\label{lem:app:parity:target}
Every end-component $E$ of $\rmc$ whose dominant color $\max_{s \in E} \coloring(s)$ is odd satisfies $E \subseteq W_{\mathrm{odd}}$. Moreover, there is a positional environment policy $\envpol_{\mathrm{stay}}$ under which $W_{\mathrm{odd}}$ is closed and, from every $s \in W_{\mathrm{odd}}$, the dominant color visited infinitely often is odd almost surely.
\end{lemma}

\begin{proof}
\emph{First part.} We show $E \subseteq W_{\mathrm{odd}}$ by induction on the recursion depth of the call $\asoddsat(B)$ that first has $E \subseteq B$ (the top-level call has $B = \states \supseteq E$). Since $E$ is an end-component contained in $B$ and MECs are maximal, $E$ is contained in some MEC $M_i$ of $B$ computed by $\mecdecomp(B)$. Let $m_i = \max_{s \in M_i} \coloring(s)$.
If $m_i$ is odd, the algorithm adds $M_i$ to its output, so $E \subseteq M_i \subseteq W_{\mathrm{odd}}$.
If $m_i$ is even, then, as $E \subseteq M_i$, its dominant color satisfies $\max_{s \in E}\coloring(s) \le m_i$; being odd it cannot equal the even $m_i$, so $\max_{s \in E}\coloring(s) < m_i$ and no state of $E$ has color $m_i$. Hence $E \cap U_i = \emptyset$ for $U_i = \{ s \in M_i : \coloring(s) = m_i \}$, i.e. $E \subseteq M_i \setminus U_i$, and the algorithm recurses on $M_i \setminus U_i$ at strictly greater depth; the inductive hypothesis gives $E \subseteq W_{\mathrm{odd}}$.

\emph{Second part.} Each set added to the output is a MEC $M_i$ (of some subset $B$) with odd dominant color $m_i$; such an $M_i$ is an end-component of $\rmc$, since every $s \in M_i$ admits $p \in \uncert_{M_i}(s) \subseteq \uncert(s)$ with $\supp(p) \subseteq M_i$. Fix any such $M_i$. Its internal transition graph $G_{M_i}$ is strongly connected, and for each $s \in M_i$ the convexity of $\uncert(s)$ lets us average the distributions witnessing the outgoing $G_{M_i}$-edges of $s$ into a single $\envpol_{\mathrm{stay}}(s) \in \uncert(s)$ with $\supp(\envpol_{\mathrm{stay}}(s)) \subseteq M_i$ that covers all $G_{M_i}$-successors of $s$. Under $\envpol_{\mathrm{stay}}$ the set $M_i$ is thus closed and strongly connected, hence a single BSCC, so a play started anywhere in $M_i$ visits every state of $M_i$ infinitely often almost surely; its dominant color is therefore $m_i$, which is odd. As $W_{\mathrm{odd}}$ is the disjoint union of these MECs, defining $\envpol_{\mathrm{stay}}$ componentwise makes $W_{\mathrm{odd}}$ closed with the stated property.
\end{proof}

\begin{lemma}[Reaching $W_{\mathrm{odd}}$]\label{lem:app:parity:bound}
For every positional environment policy $\envpol$,
\[
    \Pr\nolimits^{\envpol}_{\sinit}[\neg\,\parity(\coloring)] \;\le\; \Pr\nolimits^{\envpol}_{\sinit}[\reach(W_{\mathrm{odd}})].
\]
\end{lemma}

\begin{proof}
Fix a positional $\envpol$. The Markov chain it induces from $\sinit$ reaches a bottom strongly connected component (BSCC) almost surely. Any such BSCC $C$ is closed under $\envpol$ and strongly connected, and each $s \in C$ has $\envpol(s) \in \uncert(s)$ with $\supp(\envpol(s)) \subseteq C$, so $C$ is an end-component of $\rmc$. On the event $\neg\,\parity(\coloring)$ the play visits exactly the states of its BSCC infinitely often, i.e. $\Inf(\pi) = C$, and its dominant color $\max_{s \in C}\coloring(s)$ is odd; by Lemma~\ref{lem:app:parity:target} this forces $C \subseteq W_{\mathrm{odd}}$, so the play reaches $W_{\mathrm{odd}}$. Hence $\neg\,\parity(\coloring) \subseteq \reach(W_{\mathrm{odd}})$ up to a null set, giving the inequality.
\end{proof}

\begin{proof}[Proof of Lemma~\ref{thm:rmc:almost-sure-parity}]
By the first part of Lemma~\ref{lem:app:parity:target}, the output $W_{\mathrm{odd}} = \asoddsat(\states)$ contains every end-component with an odd dominant color; conversely, by construction it is a union of MECs $M_i$ with odd dominant color $m_i$, each an end-component of $\rmc$. Thus $W_{\mathrm{odd}}$ is exactly the set of states lying in an end-component whose dominant color is odd, and by the second part of the lemma the environment forces an odd dominant color almost surely from every state of $W_{\mathrm{odd}}$ via $\envpol_{\mathrm{stay}}$, so each such state has parity value $0$, as required.

For the complexity, we count the LP queries of the top-level call $\asoddsat(\states)$ level by level, over the two nested recursions of $\asoddsat$ and $\mecdecomp$ \emph{at once}. This is the right way to account for them: $\asoddsat$ calls $\mecdecomp$ afresh at every node of its own recursion, so composing the two bounds as black boxes ($O(|B|^3)$ queries per MEC decomposition, times an $\asoddsat$ recursion of depth $\Omega(|\states|)$) would give only $O(|\states|^4)$. But the cubic factor of Theorem~\ref{thm:rmc:mec-decomp} is not the cost of one MEC decomposition: a single invocation of $\mecdecomp(B)$ costs only $O(|B|^2)$ queries, the extra factor $|\states|$ being itself the number of levels of $\mecdecomp$'s recursion. Counting the levels of the two recursions together therefore adds these factors instead of multiplying them.

So consider the combined recursion, whose nodes are all the invocations of $\mecdecomp$ made during $\asoddsat(\states)$, each labelled by the set $B$ it is invoked on. A node $\mecdecomp(B)$ has as children either the invocations $\mecdecomp(C_1), \dots, \mecdecomp(C_k)$ on the SCCs of $B^{\star}$, or, if $\mecdecomp(B)$ returns a single MEC $M = B^{\star}$ whose dominant colour $m$ is even, the single invocation $\mecdecomp(M \setminus \coloring^{-1}(m))$ made by the recursive call of line~6 of Algorithm~\ref{algo:as-odd-sat}. This covers every invocation of $\mecdecomp$, since each MEC $M_i$ processed by $\asoddsat$ is returned by exactly one node, the one with $B^{\star} = M_i$.

There are at most $|\states|$ levels, because a child's set is a proper subset of its parent's: $C_i \subsetneq B^{\star} \subseteq B$ as $B^{\star}$ splits into $k \ge 2$ SCCs, and $M \setminus \coloring^{-1}(m) \subsetneq M \subseteq B$ as the maximal-colour states of $M$ are removed before recursing. The sets at any fixed level are pairwise disjoint, because siblings are disjoint --- they are distinct SCCs of $B^{\star}$ in the first case, and there is only one child in the second --- and children of distinct nodes stay inside the disjoint sets of their parents. Hence each level costs at most $O(|\states|^2)$ LP queries by the per-invocation bound of Theorem~\ref{thm:rmc:mec-decomp}, and $O(|\states|^3)$ queries suffice in total. The remaining work of $\asoddsat$ --- computing each $m_i$, testing its parity, and forming $M_i \setminus \coloring^{-1}(m_i)$ and the union $W$ --- is $O(|B|)$ per node, hence $O(|\states|)$ per level and $O(|\states|^2)$ in total. Each LP query has size polynomial in the encoding of $\rmc$, so $W_{\mathrm{odd}}$ is computed in polynomial time.
\end{proof}

\begin{proof}[Proof of Lemma~\ref{lem:rmc:parity}]
Write $R = \sup_{\envpol} \Pr^{\envpol}_{\sinit}[\reach(W_{\mathrm{odd}})]$ for the maximal probability of reaching $W_{\mathrm{odd}}$. Since the environment minimizes the probability of $\parity(\coloring)$,
\[
    \val^{\rmc}_{\sinit}(\parity(\coloring))
    = \inf_{\envpol} \Pr\nolimits^{\envpol}_{\sinit}[\parity(\coloring)]
    = 1 - \sup_{\envpol} \Pr\nolimits^{\envpol}_{\sinit}[\neg\,\parity(\coloring)],
\]
and this value is attained by a positional environment policy, so the last supremum may be restricted to positional policies. Indeed, the parity problem in a linearly defined RMC reduces to a parity problem in an MDP controlled by the environment: at each state $s$ let the actions be the (finitely many) vertices of the polytope $\uncert(s)$, each with the transition distribution of the corresponding corner. Every environment policy of the RMC is a policy of this MDP whose one-step choice $p \in \uncert(s)$ is a convex combination of vertex actions, and conversely; the two induce the same distributions over plays, so the parity values coincide. Since parity MDPs attain their optimal values by positional policies mapping each state to a single action, the environment has a positional optimal policy that selects one vertex of $\uncert(s)$ at each state. We show the resulting supremum equals $R$.

\emph{Upper bound.} For every positional $\envpol$, Lemma~\ref{lem:app:parity:bound} gives $\Pr^{\envpol}_{\sinit}[\neg\,\parity(\coloring)] \le \Pr^{\envpol}_{\sinit}[\reach(W_{\mathrm{odd}})] \le R$.

\emph{Lower bound.} Applying Theorem~\ref{thm:rmc:safety-poly} to the safe set $\states \setminus W_{\mathrm{odd}}$ (equivalently, to the unsafe set $W_{\mathrm{odd}}$) yields a positional environment policy $\envpol_{\mathrm{reach}}$ attaining $\Pr^{\envpol_{\mathrm{reach}}}_{\sinit}[\reach(W_{\mathrm{odd}})] = R$. Define the positional policy
\[
    \envpol^{\star}(s) =
    \begin{cases}
        \envpol_{\mathrm{stay}}(s) & s \in W_{\mathrm{odd}},\\
        \envpol_{\mathrm{reach}}(s) & s \notin W_{\mathrm{odd}},
    \end{cases}
\]
with $\envpol_{\mathrm{stay}}$ from Lemma~\ref{lem:app:parity:target}. Since $\envpol^{\star}$ agrees with $\envpol_{\mathrm{reach}}$ until $W_{\mathrm{odd}}$ is first hit, it reaches $W_{\mathrm{odd}}$ with probability $R$; and once inside, $W_{\mathrm{odd}}$ is closed under $\envpol_{\mathrm{stay}}$ and the dominant color is odd almost surely, so $\reach(W_{\mathrm{odd}})$ implies $\neg\,\parity(\coloring)$ under $\envpol^{\star}$. Hence $\Pr^{\envpol^{\star}}_{\sinit}[\neg\,\parity(\coloring)] \ge \Pr^{\envpol^{\star}}_{\sinit}[\reach(W_{\mathrm{odd}})] = R$.

Combining the two bounds, $\sup_{\envpol} \Pr^{\envpol}_{\sinit}[\neg\,\parity(\coloring)] = R$, so, using the safety--reachability duality with unsafe set $U = W_{\mathrm{odd}}$,
\[
    \val^{\rmc}_{\sinit}(\parity(\coloring)) = 1 - R = \val^{\rmc}_{\sinit}\big(\safe(\states \setminus W_{\mathrm{odd}})\big). \qedhere
\]
\end{proof}

\begin{proof}[Proof of Theorem~\ref{thm:rmc:parity-poly}]
By Lemma~\ref{thm:rmc:almost-sure-parity}, $W_{\mathrm{odd}}$ is computable in polynomial time. By Lemma~\ref{lem:rmc:parity} the parity value at every state equals the safety value of $\safe(\states \setminus W_{\mathrm{odd}})$, which, together with an optimal positional environment policy, is computed in polynomial time by Theorem~\ref{thm:rmc:safety-poly}. The policy $\envpol^{\star}$ built in the proof of Lemma~\ref{lem:rmc:parity} (reaching $W_{\mathrm{odd}}$ optimally and then applying $\envpol_{\mathrm{stay}}$) is a corresponding optimal positional environment policy for $\parity(\coloring)$.
\end{proof}

\section{Proof of Proposition~\ref{prop:polytopic} and Corollary~\ref{cor:np-conp}}
\label{app:polytopic}

This appendix proves Proposition~\ref{prop:polytopic}: every linearly defined
RMDP satisfies the positional-determinacy condition of
Proposition~\ref{prop:polytopic}. The argument reduces a linearly defined RMDP
$\rmdp$ to a finite turn-based stochastic parity game
$\vertexgame$ (the \emph{vertex game} of $\rmdp$) in which the
environment's choice of a distribution from an uncertainty polytope is
replaced by an adversarial choice of a \emph{vertex} (corner) of that
polytope followed by a probabilistic step. The reduction is
value-preserving and carries positional strategies of $\vertexgame$
back to positional policies of $\rmdp$; since turn-based stochastic
parity games are positionally determined, so is $\rmdp$. Specialising
to the single-action case yields the reduction of a linearly defined RMC to an
MDP used in Sections~\ref{sec:rmc} and~\ref{sec:RMC:parity}. Throughout
we fix a linearly defined RMDP $\rmdp=(\states,\actions,\uncert)$ and a
coloring $\coloring\colon\states\to[d]$, and write
$\spec=\parity(\coloring)$.

\subsection{Vertices of the Uncertainty Polytopes}

For a state-action pair $(s,a)$ the uncertainty set
$\uncert(s,a)=\{p\in\Delta(\states):A_{s,a}\,p\le b_{s,a}\}$ is a
bounded polytope (it is contained in the simplex $\Delta(\states)$).
Let $\vertices(\uncert(s,a))$ denote its finite set of vertices
(extreme points). The Minkowski--Weyl theorem for polytopes gives the
following standard fact, on which the whole reduction rests.

\begin{lemma}[Vertex representation]\label{lem:app:vertices}
Each polytope $\uncert(s,a)$ has finitely many vertices, and
\[
    \uncert(s,a)
    \;=\;
    \Big\{\, \textstyle\sum_{q\in\vertices(\uncert(s,a))} \lambda_q\, q
    \;:\; \lambda_q\ge 0,\ \textstyle\sum_q \lambda_q = 1 \,\Big\},
\]
i.e. every admissible distribution $p\in\uncert(s,a)$ is a convex
combination of vertices of $\uncert(s,a)$.
\end{lemma}

Since a convex combination of distributions is again a distribution,
each vertex $q\in\vertices(\uncert(s,a))$ is itself a distribution over
$\states$, and it is the extreme-point distribution the environment may
commit to. Lemma~\ref{lem:app:vertices} is the only place where
polytopicity is used: it is what makes the environment's continuous
choice reducible to a \emph{finite} one.

\subsection{The Vertex Game}

We build a turn-based stochastic parity game
$\vertexgame=(V,V_{\Diamond},V_{\Box},V_{\bigcirc},E,\coloring_{\vertexgame})$
with three kinds of vertices:
\begin{itemize}
    \item \emph{Agent vertices} $V_{\Diamond}=\{\,v_s : s\in\states\,\}$, one per state, owned by the agent (the maximiser);
    \item \emph{Environment vertices} $V_{\Box}=\{\,v_{s,a} : s\in\states,\ a\in\actions\,\}$, one per state-action pair, owned by the environment (the minimiser);
    \item \emph{Probabilistic vertices} $V_{\bigcirc}=\{\,v_{s,a,q} : s\in\states,\ a\in\actions,\ q\in\vertices(\uncert(s,a))\,\}$, one per chosen vertex of an uncertainty polytope.
\end{itemize}
The moves are as follows. From an agent vertex $v_s$ the agent chooses
an action $a\in\actions$ and moves to the environment vertex $v_{s,a}$.
From an environment vertex $v_{s,a}$ the environment chooses a vertex
$q\in\vertices(\uncert(s,a))$ and moves to the probabilistic vertex
$v_{s,a,q}$. From a probabilistic vertex $v_{s,a,q}$ the next vertex is
sampled: the game moves to the agent vertex $v_{s'}$ with probability
$q(s')$ for each $s'\in\states$. Thus
\[
    E \;=\;
    \{\,(v_s,v_{s,a}) : a\in\actions\,\}
    \cup
    \{\,(v_{s,a},v_{s,a,q}) : q\in\vertices(\uncert(s,a))\,\}
    \cup
    \{\,(v_{s,a,q},v_{s'}) : q(s')>0\,\}.
\]
The coloring $\coloring_{\vertexgame}$ assigns to each agent vertex the
color of its state, $\coloring_{\vertexgame}(v_s)=\coloring(s)$, and to
every environment and probabilistic vertex the least color $0$, which is
even and dominated by every color of $[d]$:
\[
    \coloring_{\vertexgame}(v_s)=\coloring(s),
    \qquad
    \coloring_{\vertexgame}(v_{s,a})=\coloring_{\vertexgame}(v_{s,a,q})=0 .
\]
A play of $\vertexgame$ satisfies the parity objective iff the largest
color occurring infinitely often is even. The game is finite: it has
$|\states|+|\states||\actions|+\sum_{s,a}|\vertices(\uncert(s,a))|$
vertices.

\subsection{Correspondence of Plays and Strategies}

Every play of $\vertexgame$ from $v_{\sinit}$ has the form
\[
    v_{s_0}\, v_{s_0,a_0}\, v_{s_0,a_0,q_0}\, v_{s_1}\, v_{s_1,a_1}\, v_{s_1,a_1,q_1}\, v_{s_2}\cdots ,
\]
alternating one agent vertex, one environment vertex and one
probabilistic vertex per round. Deleting the environment and
probabilistic vertices yields the \emph{projected path}
$\rho(\pi)=s_0 s_1 s_2\cdots\in\states^\omega$, an infinite path of
$\rmdp$. The projection $\rho$ is a colour-faithful correspondence.

\begin{lemma}[Play projection]\label{lem:app:play-proj}
For every play $\pi$ of $\vertexgame$, the maximal color occurring
infinitely often in $\pi$ equals the maximal color occurring infinitely
often in $\rho(\pi)$. Consequently $\pi$ satisfies $\parity(\coloring_{\vertexgame})$
iff $\rho(\pi)$ satisfies $\parity(\coloring)$.
\end{lemma}

\begin{proof}
Each round of $\pi$ contributes exactly one agent vertex $v_{s_i}$ with
color $\coloring(s_i)$ and two auxiliary vertices with color $0$. Hence
the multiset of colors seen infinitely often in $\pi$ is that of
$\rho(\pi)$ together with the color $0$ (seen infinitely often, as the
play is infinite). Since $0$ is the least color, it never changes the
maximum, so the maximal recurring color of $\pi$ equals that of
$\rho(\pi)$. The parity condition depends only on this maximum.
\end{proof}

We now match strategies. A (behavioural) agent strategy in
$\vertexgame$ chooses, given the history ending at an agent vertex
$v_s$, a distribution over actions; a positional (pure memoryless)
agent strategy is a map $\states\to\actions$. These are literally the
agent policies of $\rmdp$, with positional strategies corresponding to
positional agent policies. For the environment, at an environment
vertex $v_{s,a}$ a behavioural strategy chooses a distribution over the
vertices $\vertices(\uncert(s,a))$; the induced one-step distribution
over successor states $s'$ is
$\sum_{q}\lambda_q\, q(s')$, i.e. the convex combination
$\sum_q \lambda_q\, q\in\uncert(s,a)$ by
Lemma~\ref{lem:app:vertices}. The next lemma records that these convex
combinations range over exactly $\uncert(s,a)$, so environment
strategies of $\vertexgame$ and environment policies of $\rmdp$ realise
precisely the same one-step distributions.

\begin{lemma}[Environment moves]\label{lem:app:env-moves}
Fix $(s,a)$. The set of one-step distributions over $\states$ that the
environment can realise at $v_{s,a}$ in $\vertexgame$, namely
$\{\sum_q \lambda_q\, q : \lambda\in\Delta(\vertices(\uncert(s,a)))\}$, equals $\uncert(s,a)$. In particular, a \emph{pure} choice of a
single vertex $q$ realises the distribution $q$, and every distribution
$p\in\uncert(s,a)$ is realised by some (possibly randomised) choice.
\end{lemma}

\begin{proof}
Every convex combination of vertices lies in $\uncert(s,a)$ because the
polytope is convex; conversely, by Lemma~\ref{lem:app:vertices} every
$p\in\uncert(s,a)$ is such a convex combination. Committing to a single
vertex $q$ (i.e. $\lambda_q=1$) realises $q$ itself.
\end{proof}

Given a strategy profile in $\vertexgame$ and its counterpart in
$\rmdp$ realising the same one-step distributions at every vertex,
the two induce, via the cylinder construction, the same probability
distribution over projected paths; combined with
Lemma~\ref{lem:app:play-proj} this yields value preservation.

\begin{lemma}[Value preservation]\label{lem:app:value-pres}
For every $s\in\states$, the value of $\vertexgame$ at $v_s$ for the
objective $\parity(\coloring_{\vertexgame})$ equals
$\val^{\rmdp}_s(\spec)$. Moreover, for a fixed positional agent policy
$\agentpol$ (resp. positional environment policy $\envpol$) the value of
the corresponding one-player game equals
$\inf_{\envpol'}\Pr^{\agentpol,\envpol'}_s[\spec]$ (resp.
$\sup_{\agentpol'}\Pr^{\agentpol',\envpol}_s[\spec]$).
\end{lemma}

\begin{proof}
Fix any agent strategy and environment strategy in $\vertexgame$. By
Lemma~\ref{lem:app:env-moves} there are an agent policy and an
environment policy of $\rmdp$ that at every state (resp. state-action
pair) reproduce the same one-step distribution over successors, and
conversely every pair of policies of $\rmdp$ arises this way. Because
the one-step distributions agree at every vertex, the two induced
measures assign equal probability to every cylinder of projected paths,
so $\Pr$ of the event $\{\pi:\rho(\pi)\models\parity(\coloring)\}$ in
$\vertexgame$ equals $\Pr^{\agentpol,\envpol}_s[\spec]$ in $\rmdp$; by
Lemma~\ref{lem:app:play-proj} the former equals the probability of
$\parity(\coloring_{\vertexgame})$. Taking $\sup$ over agent strategies
and $\inf$ over environment strategies on both sides gives
$\mathrm{val}_{\vertexgame}(v_s)=\sup_{\agentpol}\inf_{\envpol}\Pr^{\agentpol,\envpol}_s[\spec]=\val^{\rmdp}_s(\spec)$.
The one-player statements follow by fixing the corresponding player's
strategy to the positional policy and repeating the argument.
\end{proof}

\subsection{Proof of Proposition~\ref{prop:polytopic}}

\begin{proof}[Proof of Proposition~\ref{prop:polytopic}]
Turn-based stochastic games with a parity objective are positionally
(pure-memoryless) determined: both players have optimal pure memoryless
strategies, and the game has a value satisfying
$\sup_{\agentpol}\inf_{\envpol}=\inf_{\envpol}\sup_{\agentpol}$
\cite{chatterjee2012survey}. Apply this to the vertex game
$\vertexgame$ of $\rmdp$. Let $\hat{\agentpol}$ and $\hat{\envpol}$ be
optimal pure memoryless strategies of the agent and the environment in
$\vertexgame$.

Map them back to $\rmdp$. The strategy $\hat{\agentpol}$ selects, at
each agent vertex $v_s$, a single action $\agentpol^\ast(s)\in\actions$;
this is a positional agent policy of $\rmdp$. The strategy
$\hat{\envpol}$ selects, at each environment vertex $v_{s,a}$, a single
vertex $q_{s,a}\in\vertices(\uncert(s,a))$; setting
$\envpol^\ast(s,a)=q_{s,a}\in\uncert(s,a)$ (a vertex distribution) gives
a positional environment policy of $\rmdp$. By
Lemma~\ref{lem:app:env-moves} these choices realise exactly the same
one-step distributions in $\rmdp$ as $\hat{\agentpol},\hat{\envpol}$ do
in $\vertexgame$, so by Lemma~\ref{lem:app:value-pres} they are value
preserving.

Optimality of
$\hat{\agentpol}$ in $\vertexgame$ means that against every environment
strategy it secures at least the game value, which by
Lemma~\ref{lem:app:value-pres} is $\val^{\rmdp}_s(\spec)$; translating
through the strategy correspondence,
$\inf_{\envpol}\Pr^{\agentpol^\ast,\envpol}_s[\spec]=\val^{\rmdp}_s(\spec)$
for all $s$. Symmetrically, optimality of $\hat{\envpol}$ gives
$\sup_{\agentpol}\Pr^{\agentpol,\envpol^\ast}_s[\spec]=\val^{\rmdp}_s(\spec)$
for all $s$. Hence
\[
    \inf_{\envpol}\Pr^{\agentpol^\ast,\envpol}_s[\spec]
    \;=\;\val^{\rmdp}_s(\spec)\;=\;
    \sup_{\agentpol}\Pr^{\agentpol,\envpol^\ast}_s[\spec]
    \qquad\text{for all }s\in\states,
\]
with $\agentpol^\ast$ and $\envpol^\ast$ positional, which is exactly
the positional-determinacy condition. In
particular the two orders of optimization coincide, so $\rmdp$ is
determined. This proves Proposition~\ref{prop:polytopic}.
\end{proof}

\subsection{Proof of Corollary~\ref{cor:np-conp}}
\label{app:np-conp}

Throughout this subsection, \textsc{RMDP-Parity} denotes the decision
problem of Corollary~\ref{cor:np-conp}: given a linearly defined RMDP
$\rmdp=(\states,\actions,\uncert)$, a coloring
$\coloring\colon\states\to[d]$, a state $s\in\states$ and a rational
threshold $\lambda\in[0,1]$ written in binary, decide whether
$\val^{\rmdp}_s(\parity(\coloring))\ge\lambda$. Analogously,
\textsc{SG-Parity} denotes the problem of deciding, for a finite
turn-based stochastic parity game $\mathcal{H}$ with rational transition
probabilities, a vertex $v$ of $\mathcal{H}$ and a rational
$\lambda\in[0,1]$, whether $\mathrm{val}_{\mathcal{H}}(v)\ge\lambda$. We
write $\|\rmdp\|$ for the bit-size of the encoding of $\rmdp$, and recall
from the preliminaries that each uncertainty set is presented as
\[
    \uncert(s,a)=\Big\{x\in\Delta(\states)\;\Big|\;\exists y\in\mathbb{R}^{k}\colon
    A_{s,a}\begin{bmatrix}x\\y\end{bmatrix}\le b_{s,a}\Big\},
\]
where we assume without loss of generality that the simplex constraints
defining $\Delta(\states)$ are among the rows of $A_{s,a}\,[x;y]\le
b_{s,a}$. The two assertions of Corollary~\ref{cor:np-conp} are proved
separately in Lemma~\ref{lem:app:np-conp} and
Lemma~\ref{lem:app:sg-hardness} below.

\subsubsection{Membership in $\mathit{NP}\cap\mathit{coNP}$.}

The certificates we use are positional policies. For the agent a
positional policy is a map $\states\to\actions$ and is trivially of size
$O(|\states|\log|\actions|)$. For the environment a positional policy
must name a distribution inside each uncertainty set, so we first record
that the distributions selected by the optimal environment policy
constructed in the proof of Proposition~\ref{prop:polytopic}
(vertices of the uncertainty polytopes) admit a compact encoding, even
though the number of such vertices may be exponential
(Remark~\ref{rem:app:size}).

\begin{lemma}[Vertices admit compact encodings]\label{lem:app:vertex-size}
Let $\uncert(s,a)$ be a linearly defined uncertainty set. Then every
$q\in\vertices(\uncert(s,a))$ is a rational vector of bit-size polynomial
in $\|\rmdp\|$. Moreover, given a rational vector $q\in\mathbb{Q}^{\states}$,
one can decide in time polynomial in $\|\rmdp\|$ and the bit-size of $q$
whether $q\in\uncert(s,a)$.
\end{lemma}

\begin{proof}
Write $A_{s,a}=[A^{x}\mid A^{y}]$ according to the split of the variables
into $x$ and $y$, and let
$Q=\{(x,y):A^{x}x+A^{y}y\le b_{s,a}\}$, so that $\uncert(s,a)$ is the
projection of $Q$ onto the $x$-coordinates. By the Fourier--Motzkin
projection lemma,
\[
    \uncert(s,a)=\{\,x\in\mathbb{R}^{\states}\;:\;
    (\mu^{\top}A^{x})\,x\le\mu^{\top}b_{s,a}\ \text{ for all }\mu\in K\,\},
\]
where $K=\{\mu\ge 0:\mu^{\top}A^{y}=0\}$ is the \emph{projection cone} of
$Q$, and it suffices to range over the finitely many extreme rays of $K$.
Each extreme ray of $K$ is a basic solution of a rational system whose
size is bounded by $\|\rmdp\|$, hence, by Cramer's rule, is a rational
vector of bit-size polynomial in $\|\rmdp\|$; consequently each
inequality $(\mu^{\top}A^{x})x\le\mu^{\top}b_{s,a}$ in the above
description has bit-size polynomial in $\|\rmdp\|$, although there may be
exponentially many of them. Now let $q$ be a vertex of $\uncert(s,a)$.
Since $\uncert(s,a)\subseteq\Delta(\states)$ is a polytope of dimension at
most $|\states|$, the vertex $q$ is the unique solution of a subsystem of
$|\states|$ linearly independent inequalities of the above description
taken with equality. Applying Cramer's rule once more to this subsystem,
whose entries are of bit-size polynomial in $\|\rmdp\|$, shows that $q$ is
rational of bit-size polynomial in $\|\rmdp\|$.

For the second claim, $q\in\uncert(s,a)$ holds iff the linear system
$A^{y}y\le b_{s,a}-A^{x}q$ in the unknown $y$ is feasible, which is an LP
feasibility test on rational data of size polynomial in $\|\rmdp\|$ and
the size of $q$, and is therefore decidable in polynomial time.
\end{proof}

\begin{lemma}\label{lem:app:np-conp}
$\textsc{RMDP-Parity}\in\mathit{NP}\cap\mathit{coNP}$.
\end{lemma}

\begin{proof}
We show that both \textsc{RMDP-Parity} and its complement are in
$\mathit{NP}$.

\emph{Membership in $\mathit{NP}$.} The certificate is a positional agent
policy $\agentpol\colon\states\to\actions$, of size polynomial in
$\|\rmdp\|$. The verifier constructs the induced RMC
$\rmdp^{\agentpol}=(\states,\uncert^{\agentpol})$ with
$\uncert^{\agentpol}(s)=\uncert(s,\agentpol(s))$; this is again linearly
defined, is obtained from $\rmdp$ by discarding the uncertainty sets of
the unselected actions, and hence is computable in polynomial time. By
Theorem~\ref{thm:rmc:parity-poly} the verifier then computes
$\val^{\agentpol}_s=\inf_{\envpol}\Pr^{\agentpol,\envpol}_s[\spec]$ in
time polynomial in $\|\rmdp\|$; in particular $\val^{\agentpol}_s$ is a
rational of bit-size polynomial in $\|\rmdp\|$, being the output of a
polynomial-time computation, so the verifier can compare it with
$\lambda$ exactly in polynomial time. It accepts iff
$\val^{\agentpol}_s\ge\lambda$.

Soundness: for every positional agent policy $\agentpol$ we have
$\val^{\agentpol}_s\le\sup_{\agentpol'}\inf_{\envpol}\Pr^{\agentpol',\envpol}_s[\spec]=\val^{\rmdp}_s(\spec)$,
so an accepting certificate witnesses
$\val^{\rmdp}_s(\spec)\ge\lambda$. Completeness: if
$\val^{\rmdp}_s(\spec)\ge\lambda$, then the positional policy
$\agentpol^{\ast}$ of Proposition~\ref{prop:polytopic} satisfies
$\val^{\agentpol^{\ast}}_s=\val^{\rmdp}_s(\spec)\ge\lambda$ and is
accepted.

\emph{Membership in $\mathit{coNP}$.} We give an $\mathit{NP}$ procedure
for the complement, i.e. for deciding
$\val^{\rmdp}_s(\spec)<\lambda$. The certificate is a positional
environment policy $\envpol$, presented as a table of rational
distributions $\envpol(s,a)\in\uncert(s,a)$ for all
$(s,a)\in\states\times\actions$; the verifier first checks
$\envpol(s,a)\in\uncert(s,a)$ for every pair, which by
Lemma~\ref{lem:app:vertex-size} takes polynomial time, and rejects the
certificate otherwise. Fixing $\envpol$ leaves the MDP
$\rmdp^{\envpol}=(\states,\actions,\trans)$ with
$\trans(s,a)=\envpol(s,a)$, whose transition probabilities are exactly
the entries of the certificate and whose size is therefore polynomial in
the size of the certificate. Since agent policies of $\rmdp$ played
against the fixed $\envpol$ are precisely the policies of
$\rmdp^{\envpol}$, we have
$\sup_{\agentpol}\Pr^{\agentpol,\envpol}_s[\spec]=\val^{\rmdp^{\envpol}}_s(\parity(\coloring))$.
The quantitative parity problem on MDPs is solvable in polynomial
time~\cite{chatterjee2012survey}: one computes the maximal
end-component decomposition, retains those end-components in which the
agent can enforce an even dominant colour, and solves a maximal
reachability LP for their union; in particular the value is a rational of
bit-size polynomial in the size of $\rmdp^{\envpol}$. The verifier
computes this value and accepts iff it is smaller than $\lambda$.

Soundness: for every positional environment policy $\envpol$ we have
$\sup_{\agentpol}\Pr^{\agentpol,\envpol}_s[\spec]\ge\sup_{\agentpol}\inf_{\envpol'}\Pr^{\agentpol,\envpol'}_s[\spec]=\val^{\rmdp}_s(\spec)$,
so an accepting certificate witnesses
$\val^{\rmdp}_s(\spec)<\lambda$. Completeness: if
$\val^{\rmdp}_s(\spec)<\lambda$, take the positional environment policy
$\envpol^{\ast}$ constructed in the proof of
Proposition~\ref{prop:polytopic}. It satisfies
$\sup_{\agentpol}\Pr^{\agentpol,\envpol^{\ast}}_s[\spec]=\val^{\rmdp}_s(\spec)<\lambda$
and, crucially, it selects at each pair $(s,a)$ a \emph{vertex}
$q_{s,a}\in\vertices(\uncert(s,a))$; by
Lemma~\ref{lem:app:vertex-size} each $q_{s,a}$ is rational of bit-size
polynomial in $\|\rmdp\|$, so $\envpol^{\ast}$ is a certificate of
polynomial size and is accepted.
\end{proof}

We stress that neither direction materialises the vertex game
$\vertexgame$: the $\mathit{NP}$ verifier runs the polynomial-time LP-based
RMC solver of Section~\ref{sec:rmc} on the compact facet description, and
the $\mathit{coNP}$ verifier only ever writes down one vertex per
state-action pair rather than all of them.

\subsubsection{Hardness.}

For the lower bound we exhibit the converse of the vertex-game
construction: stochastic games are a special case of linearly defined
RMDPs, in which the environment's uncertainty set is the convex hull of
the Dirac distributions on the available successors.

\begin{lemma}\label{lem:app:sg-hardness}
$\textsc{SG-Parity}$ reduces to $\textsc{RMDP-Parity}$ in polynomial
time. Hence $\textsc{RMDP-Parity}$ is at least as hard as
$\textsc{SG-Parity}$.
\end{lemma}

\begin{proof}
Let $\mathcal{H}=(V,V_{\Diamond},V_{\Box},V_{\bigcirc},E,\coloring_{\mathcal{H}})$
be a turn-based stochastic parity game with a rational transition
distribution $p_{v}\in\Delta(V)$ at every
$v\in V_{\bigcirc}$, together with a vertex $v_0$ and a rational
$\lambda$. We build a linearly defined RMDP
$\rmdp_{\mathcal{H}}=(V,\actions,\uncert)$ with
$\actions=\{a_1,\dots,a_{m}\}$, where $m$ is the maximum
out-degree in $\mathcal{H}$, and coloring
$\coloring=\coloring_{\mathcal{H}}$, as follows. Fix for each vertex $v$
an enumeration $v^{1},\dots,v^{\deg(v)}$ of its successors.
\begin{compactitem}
\item For $v\in V_{\Diamond}$ (an agent vertex) and $i\le\deg(v)$, put
$\uncert(v,a_i)=\{\delta_{v^{i}}\}$, the singleton containing the Dirac
distribution on the $i$-th successor.
\item For $v\in V_{\Box}$ (an environment vertex), put
$\uncert(v,a_i)=\mathrm{conv}\{\delta_{v^{1}},\dots,\delta_{v^{\deg(v)}}\}
=\{p\in\Delta(V):p(u)=0\text{ for all }u\text{ with }(v,u)\notin E\}$
for every $i$.
\item For $v\in V_{\bigcirc}$ (a probabilistic vertex), put
$\uncert(v,a_i)=\{p_{v}\}$ for every $i$.
\end{compactitem}
For $v\in V_{\Diamond}$ and $i>\deg(v)$ we set
$\uncert(v,a_i)=\uncert(v,a_1)$, so that the action set is uniform and no
action is spurious. Every uncertainty set above is cut out by linear
equalities and the simplex constraints, hence is linearly defined without
auxiliary variables, and $\rmdp_{\mathcal{H}}$ is computable from
$\mathcal{H}$ in polynomial time.

It remains to show
$\val^{\rmdp_{\mathcal{H}}}_{v}(\parity(\coloring))=\mathrm{val}_{\mathcal{H}}(v)$
for every $v\in V$, from which the reduction follows by asking
\textsc{RMDP-Parity} about $\rmdp_{\mathcal{H}}$, $v_0$ and $\lambda$.
Consider the vertex game $\vertexgame$ of $\rmdp_{\mathcal{H}}$
constructed above. For $v\in V_{\Box}$ the distributions
$\delta_{v^{1}},\dots,\delta_{v^{\deg(v)}}$ are distinct unit vectors and
therefore affinely independent, so each of them is an extreme point of
their convex hull and no other point is:
$\vertices(\uncert(v,a_i))=\{\delta_{v^{1}},\dots,\delta_{v^{\deg(v)}}\}$.
For $v\in V_{\Diamond}\cup V_{\bigcirc}$ the uncertainty sets are
singletons and thus have a single vertex. Hence in $\vertexgame$: at
$v\in V_{\Diamond}$ the agent chooses one of the successors of $v$ and
the play proceeds deterministically to it; at $v\in V_{\Box}$ the agent
has no meaningful choice, the environment chooses a vertex
$\delta_{v^{i}}$, and the play proceeds deterministically to $v^{i}$; at
$v\in V_{\bigcirc}$ neither player has a choice and the next vertex is
sampled from $p_{v}$. Thus $\vertexgame$ is
exactly $\mathcal{H}$ with every edge subdivided by at most two auxiliary
vertices of colour $0$, and with a dummy choice inserted at the vertices
of $V_{\Box}\cup V_{\bigcirc}$. Since $0$ is dominated by every colour of
$[d]$, the argument of Lemma~\ref{lem:app:play-proj} applies verbatim:
the maximal colour occurring infinitely often along a play of
$\vertexgame$ equals that of the corresponding play of $\mathcal{H}$, and
the strategy correspondence of Lemma~\ref{lem:app:value-pres} maps
strategies of the two games to one another while preserving the induced
distributions over projected paths. Consequently
$\mathrm{val}_{\vertexgame}(v)=\mathrm{val}_{\mathcal{H}}(v)$, and by
Lemma~\ref{lem:app:value-pres} applied to $\rmdp_{\mathcal{H}}$ we get
$\val^{\rmdp_{\mathcal{H}}}_{v}(\parity(\coloring))=\mathrm{val}_{\vertexgame}(v)=\mathrm{val}_{\mathcal{H}}(v)$,
as required.
\end{proof}

Lemma~\ref{lem:app:np-conp} and Lemma~\ref{lem:app:sg-hardness} together
prove Corollary~\ref{cor:np-conp}. Two consequences are worth recording.
First, since no polynomial-time algorithm for \textsc{SG-Parity} is
known (the best known bounds are randomised
sub-exponential~\cite{chatterjee2005algorithms}) a polynomial-time
algorithm for the quantitative parity problem on linearly defined RMDPs
would resolve a long-standing open problem; this is the sense in which
the exponential iteration bound of Theorem~\ref{cor:total} is not an
artefact of our analysis. Second, membership in
$\mathit{NP}\cap\mathit{coNP}$ implies that \textsc{RMDP-Parity} is not
$\mathit{NP}$-hard unless $\mathit{NP}=\mathit{coNP}$, so the source of
hardness is not $\mathit{NP}$-hardness but the same
``$\mathit{NP}\cap\mathit{coNP}$ barrier'' as for stochastic games.

\begin{remark}[The RMC specialisation]\label{rem:app:rmc-mdp}
When $|\actions|=1$ the agent has no choice and the agent vertices merge
with the (unique) environment vertices, so $\vertexgame$ becomes a
\emph{one-player} stochastic game controlled by the environment, i.e. an
MDP whose actions at state $s$ are the vertices $\vertices(\uncert(s))$
and whose transition under vertex $q$ is $q$ itself. Every environment
policy of the RMC is a policy of this MDP whose one-step choice is a
convex combination of vertex actions, and conversely, so the two induce
the same distributions over plays and the parity values coincide
(Lemma~\ref{lem:app:value-pres}). Since parity MDPs admit optimal
positional policies, the environment has a positional optimal policy
selecting one vertex of $\uncert(s)$ at each state. This is precisely
the reduction invoked in the proof of Lemma~\ref{lem:rmc:parity} and
in Section~\ref{sec:RMC:parity}.
\end{remark}

\begin{remark}[Size of the reduction]\label{rem:app:size}
The vertex game $\vertexgame$ is value-equivalent to $\rmdp$ and could
in principle be solved directly by a stochastic-game (resp. MDP) parity
solver, computing $\val^{\rmdp}(\spec)$ exactly. However its size is
governed by $\sum_{s,a}|\vertices(\uncert(s,a))|$, and a polytope
described by polynomially many facets in dimension $|\states|$ can have
$|\vertices(\uncert(s,a))|$ exponential in $|\states|$ (for instance the
$L_1$ ball, whose vertex count grows with $|\states|$). Thus
$\vertexgame$ may be exponentially larger than $\rmdp$, and even writing
it down need not be possible in polynomial space. This is exactly the
inefficiency that motivates the polynomial time and space, LP-based algorithms of
Section~\ref{sec:rmc} and Section~\ref{sec:rmdp}, which operate on the compact facet description of
the polytopes rather than on their vertices.
\end{remark}

\section{Proofs of Section~\ref{sec:rmdp}} \label{app:rmdp}

This appendix provides the proofs of the results of
Section~\ref{sec:rmdp}. Throughout we use the notation and standing
objects of that section: the linearly defined RMDP
$\rmdp=(\states,\actions,\uncert)$, the coloring $\coloring$ and
objective $\spec=\parity(\coloring)$, the induced RMC
$\rmdp^{\agentpol}$ and its value vector $\val^{\agentpol}$, the value
classes $\vc{\agentpol}{r}$, the one-step optima
$\nu(s,a)=\min_{p\in\uncert(s,a)}p\cdot\val^{\agentpol}$, tight actions
$\tightacts{\val^{\agentpol}}{s}$, tight faces
$\tightset{\val^{\agentpol}}{s}{a}$, the improvable set $I$, and the trap
value-class RMDPs $\rmdp_r$ over $U_r\cup\{\bot\}$ (where the agent is
restricted to tight actions, the environment to tight faces, and all mass
leaving $U_r$ is collapsed onto the losing sink $\bot$ by the linear map
$\kappa_r$) together with their colorings $\coloring_r$, almost-sure
regions $W_r=\AS_{\rmdp_r}(\parity(\coloring_r))\cap U_r$ and positional
almost-sure winning policies $\agentpol_r$ on $W_r$. Recall that
$\agentpol_r(s)\in\tightacts{\val^{\agentpol}}{s}$ for every $s\in W_r$
by construction, since $\rmdp_r$ offers no other actions.

We use freely the fact, established in Section~\ref{sec:rmdp}, that
$p\cdot\uvector=\uvector_s$ for every tight action $a\in A^{\uvector}(s)$
and every $p$ in the tight face $\tightset{\uvector}{s}{a}$, the support property of $W_r$ recorded in
Lemma~\ref{lem:as-facts} below, and the elementary facts about one-step
minima: since each uncertainty set is a nonempty compact polytope and
$p\mapsto p\cdot\uvector$ is linear, every infimum
$\inf_{p\in\uncert(s,a)}p\cdot\uvector$ is attained, and the set of
minimisers is the nonempty face of $\uncert(s,a)$ obtained by adjoining
the single linear inequality $p\cdot\uvector\le\nu_{\uvector}(s,a)$,
whose vertices are vertices of $\uncert(s,a)$. In particular every
$\arg\min$ below is nonempty, and so is every uncertainty set of
$\rmdp_r$, being the image of a nonempty set under $\kappa_r$. Hence,
every $\rmdp_r$ is a legitimate linearly defined RMDP and
Proposition~\ref{prop:polytopic} applies to it.

\paragraph{Markov-chain preliminaries.}
When both policies are positional the pair $(\agentpol,\envpol)$
induces a finite Markov chain over $\states$ with transition matrix
$P(s,t)=\envpol(s,\agentpol(s))(t)$. We use the standard theory of
finite Markov chains \cite{model-checking-book}: with probability $1$ the play
eventually enters a closed recurrent class $\recurrent$ and visits all
its states infinitely often, so $\Inf(\pi)=\recurrent$ almost surely
conditioned on entering $\recurrent$. Call $\recurrent$ \emph{even} if
$\max\coloring(\recurrent)$ is even and \emph{odd} otherwise;
conditioned on entering $\recurrent$, the play satisfies $\spec$ with
probability $1$ if $\recurrent$ is even and $0$ if it is odd. Hence,
writing $\lozenge\recurrent$ for the event of being absorbed in
$\recurrent$,
\begin{equation}\label{eq:absorption}
  \Pr^{\agentpol,\envpol}_s[\spec]
  \;=\;\sum_{\recurrent\ \text{even}}
  \Pr^{\agentpol,\envpol}_s[\lozenge\recurrent].
\end{equation}
We also record the \emph{harmonicity} of the satisfaction probability of
a prefix-independent objective under a positional pair: writing
$w_s:=\Pr^{\agentpol,\envpol}_s[\spec]$,
\begin{equation}\label{eq:harmonic}
  w_s\;=\;\sum_{t\in\states}P(s,t)\,w_t,
  \qquad\text{i.e.}\qquad w=Pw .
\end{equation}
This is immediate from the Markov property together with
prefix-independence of $\spec$: conditioned on the first transition
$s\to t$, the remaining play is distributed as a play from $t$ under the
same positional pair, and deleting the one-step prefix does not change
whether the play satisfies $\spec$.

\subsection{Optimality conditions}

Proposition~\ref{prop:polytopic} yields the usual one-step optimality
(Bellman) conditions. Note that, by compactness of the uncertainty sets,
the inner operators below are minima and not merely infima.

\begin{proposition}[Optimality conditions]\label{prop:bellman}
For every $s\in\states$,
\begin{equation}\label{eq:bellman}
  \val^{\rmdp}_s(\spec)\;=\;\max_{a\in\actions}\ \min_{p\in\uncert(s,a)}
  \ p\cdot\val^{\rmdp}(\spec).
\end{equation}
Moreover, for the RMC $\rmdp^{\agentpol}$ induced by a positional agent
policy $\agentpol$,
\begin{equation}\label{eq:rmc-bellman}
  \val^{\agentpol}_s\;=\;\min_{p\in\uncert(s,\agentpol(s))}
  p\cdot\val^{\agentpol}\qquad\text{for every }s\in\states,
\end{equation}
and the minimum in \eqref{eq:rmc-bellman} is attained
\emph{simultaneously at all states} by the positional environment policy
of Proposition~\ref{prop:polytopic}.
\end{proposition}

\begin{proof}[Proof of Proposition~\ref{prop:bellman}]
Both identities peel off the first transition out of $s$ and identify
the value of the remaining play with the value vector again. This is
legitimate because parity is \emph{prefix-independent}, i.e. whether a
path satisfies $\parity(\coloring)$ depends only on $\Inf(\pi)$, so
deleting or prepending a finite prefix never changes membership.

\emph{Proof of \eqref{eq:rmc-bellman}.} Write $u:=\val^{\agentpol}$. By
the last clause of Proposition~\ref{prop:polytopic} there is a
positional environment policy $\envpol^\ast$ of $\rmdp^{\agentpol}$ with
$\Pr^{\envpol^\ast}_t[\spec]=u_t$ for every $t$. Applying
\eqref{eq:harmonic} to the chain it induces gives
$u_s=\envpol^\ast(s)\cdot u\ge\min_{p\in\uncert(s,\agentpol(s))}p\cdot
u$. Conversely, fix $p\in\uncert(s,\agentpol(s))$ and let $\envpol_p$ be
the environment policy that plays $p$ on the first step at $s$ and
follows $\envpol^\ast$ afterwards. It is admissible because
$p\in\uncert(s,\agentpol(s))$. By prefix-independence,
$\Pr^{\envpol_p}_s[\spec]
=\sum_{t}p(t)\,\Pr^{\envpol^\ast}_t[\spec]=p\cdot u$, whence
$u_s=\inf_{\envpol}\Pr^{\envpol}_s[\spec]\le p\cdot u$. Minimising over
$p$ gives $u_s\le\min_p p\cdot u$, and the two inequalities give
\eqref{eq:rmc-bellman}. The final clause of the statement is the first
display: $\envpol^\ast(s)$ attains the minimum at every $s$ at once.

\emph{Proof of \eqref{eq:bellman}.} Write $v:=\val^{\rmdp}(\spec)$.

$(\le)$ By Proposition~\ref{prop:polytopic} there is a positional
agent policy $\agentpol^\ast$ with $\val^{\agentpol^\ast}=v$. Applying
\eqref{eq:rmc-bellman} to $\agentpol^\ast$ at $s$, with
$a^\ast=\agentpol^\ast(s)$,
\[
  v_s=\min_{p\in\uncert(s,a^\ast)}p\cdot v
  \;\le\;\max_{a\in\actions}\min_{p\in\uncert(s,a)}p\cdot v .
\]

$(\ge)$ By Proposition~\ref{prop:polytopic} fix a positional
environment policy $\envpol^\ast$ with
$\sup_{\agentpol}\Pr^{\agentpol,\envpol^\ast}_t[\spec]=v_t$ for every
$t$. Fixing $\envpol^\ast$ turns $\rmdp$ into an ordinary MDP with
transition function $\delta(t,a)=\envpol^\ast(t,a)$ whose parity value at
$t$ is $v_t$. This MDP has a positional
optimal agent policy $\bar{\agentpol}$, so
$\Pr^{\bar{\agentpol},\envpol^\ast}_t[\spec]=v_t$ for every $t$. Fix
$a\in\actions$ and let $\agentpol_a$ play $a$ on the first step at $s$
and follow $\bar{\agentpol}$ afterwards. By prefix-independence,
\[
  \Pr^{\agentpol_a,\envpol^\ast}_s[\spec]
  =\sum_{t\in\states}\envpol^\ast(s,a)(t)\,
    \Pr^{\bar{\agentpol},\envpol^\ast}_t[\spec]
  =\envpol^\ast(s,a)\cdot v
  \;\ge\;\min_{p\in\uncert(s,a)}p\cdot v .
\]
Since $v_s=\sup_{\agentpol}\Pr^{\agentpol,\envpol^\ast}_s[\spec]
\ge\Pr^{\agentpol_a,\envpol^\ast}_s[\spec]$, we get
$v_s\ge\min_{p\in\uncert(s,a)}p\cdot v$ for every $a$, and taking the
maximum over the finite set $\actions$ gives $(\ge)$.
\end{proof}

\subsection{Formal Definition of Value-Class RMDPs}
\label{sec:app:value-class}

Throughout this subsection we fix a positional agent policy $\agentpol$,
write $\uvector:=\val^{\agentpol}$ for the value vector of the induced
RMC $\rmdp^{\agentpol}$, and fix a value $r$ of $\uvector$ with
$U_r:=\vc{\agentpol}{r}\neq\emptyset$. Recall from
Section~\ref{sec:rmdp} that an action $a$ is \emph{tight} at $s$ if
$\nu(s,a)=\uvector_s$, that $\tightacts{\uvector}{s}$ denotes the set of
tight actions at $s$, and that $\tightset{\uvector}{s}{a}$ denotes the
\emph{tight face} at $(s,a)$.

The construction of $\rmdp_r$ restricts $\rmdp$ to the states of $U_r$,
to the tight actions, and to the tight faces, and merges everything
outside $U_r$ into a single fresh state. The merging is carried out by
the following map.

\begin{definition}[Collapse map]
\label{def:collapse}
Let $\bot\notin\states$ be a fresh state and let
$c_r:\states\to U_r\cup\{\bot\}$ fix every state of $U_r$ and send every
state of $\states\setminus U_r$ to $\bot$. The \emph{collapse map}
$\kappa_r$ is the push-forward of distributions along $c_r$, i.e.\
$\kappa_r:\Delta(\states)\to\Delta(U_r\cup\{\bot\})$ with
\[
  \kappa_r(p)(t)=p(t)\quad(t\in U_r),
  \qquad
  \kappa_r(p)(\bot)=\!\!\sum_{t\in\states\setminus U_r}\!\!p(t).
\]
We extend $\kappa_r$ to sets of distributions by taking images,
$\kappa_r(X)=\{\kappa_r(p)\mid p\in X\}$.
\end{definition}

Thus $\kappa_r$ leaves the mass inside the value class where it is and
lumps all remaining mass onto $\bot$. The two cases of
Definition~\ref{def:collapse} are exhaustive and disjoint, so
$\kappa_r(p)$ is again a distribution. Note that $\kappa_r$ is the
restriction to $\Delta(\states)$ of a linear map
$\mathbb{R}^{\states}\to\mathbb{R}^{U_r\cup\{\bot\}}$; in particular it
maps polytopes to polytopes, and it is in general not injective, since it
forgets how the escaping mass is distributed over $\states\setminus U_r$.

\begin{definition}[Value-class RMDP]
\label{def:value-class-rmdp}
The \emph{value-class RMDP} of $\agentpol$ at $r$ is the RMDP
$\rmdp_r=(\states_r,\actions_r,\uncert_r)$ with coloring
$\coloring_r:\states_r\to[d+1]$ given by:
\begin{enumerate}
  \item $\states_r=U_r\uplus\{\bot\}$ and
    $\actions_r=\actions\uplus\{a_\bot\}$;
  \item for $s\in U_r$ and $a\in\tightacts{\uvector}{s}$,
    \[
      \uncert_r(s,a)=\kappa_r\!\left(\tightset{\uvector}{s}{a}\right);
    \]
  \item $\uncert_r(s,a)=\{\delta_\bot\}$ for $s\in U_r$ and
    $a\notin\tightacts{\uvector}{s}$, and
    $\uncert_r(\bot,a)=\{\delta_\bot\}$ for every $a\in\actions_r$;
  \item $\coloring_r(s)=\coloring(s)$ for $s\in U_r$, and
    $\coloring_r(\bot)=2\lceil d/2\rceil+1$.
\end{enumerate}
\end{definition}

The state $\bot$ is absorbing and carries an odd color dominating every
color of $[d]$, so any play reaching $\bot$ visits only $\bot$ from then
on and violates $\parity(\coloring_r)$. Leaving the value class is
therefore losing for the agent in $\rmdp_r$, irrespective of the values
in $\rmdp$ of the states left for; likewise, by item~(3), playing a
non-tight action is losing. Consequently an optimal agent policy of
$\rmdp_r$ plays only tight actions on its almost-sure winning region, and
we may and do treat $\tightacts{\uvector}{s}$ as the actions genuinely
available at $s\in U_r$.

\begin{lemma}[Well-definedness]
\label{lem:mr-welldefined}
Every uncertainty set of $\rmdp_r$ is a nonempty polytope, and $\rmdp_r$
is a linearly defined RMDP of description size $O(|\rmdp|)$. Moreover
$\tightacts{\uvector}{s}\neq\emptyset$ for every $s\in U_r$, and
$\sum_{r}|\rmdp_r|=O(|\rmdp|)$, the sum ranging over the values of
$\uvector$.
\end{lemma}

\begin{proof}
The policy $\agentpol$ cannot be improved quantitatively, so, the action $\agentpol(s)$ is tight at every
$s\in\states$, so $\tightacts{\uvector}{s}\neq\emptyset$. For
$a\in\tightacts{\uvector}{s}$ the tight face
$\tightset{\uvector}{s}{a}$ is the face of $\uncert(s,a)$ cut out by the
single inequality $p\cdot\uvector\le\nu(s,a)$, hence a nonempty polytope,
and $\uncert_r(s,a)$ is its image under the linear map $\kappa_r$, hence
again a nonempty polytope. For linear definability,
\[
  \uncert_r(s,a)=\Big\{q\ \Big|\
    \exists p:\ A_{s,a}\,p\le b_{s,a},\
    p\cdot\uvector\le\uvector_s,\
    q=\kappa_r(p)\Big\},
\]
which presents $\uncert_r(s,a)$ as the projection onto the
$q$-coordinates of a polyhedron in the variables $(q,p)$, obtained from
the description of $\uncert(s,a)$ by adjoining one inequality and the
defining equalities of $\kappa_r$. Note that only $\uvector$ enters this
description; the one-step optima $\nu(s,a)$ are used solely in the
tightness test of items~(2) and~(3). Hence $\rmdp_r$ has description
size $O(|\rmdp|)$. Finally the value classes are pairwise disjoint and
$\states_r=U_r\uplus\{\bot\}$, so
$\sum_{r}|\states_r|\le|\states|+|\states|$ and the total description
size of all value-class RMDPs is $O(|\rmdp|)$.
\end{proof}

Since Lemma~\ref{lem:mr-welldefined} makes $\rmdp_r$ a linearly defined
RMDP, Proposition~\ref{prop:polytopic} applies to it, as do the results
of Section~\ref{sec:rmc} instantiated at a fixed positional agent policy.
We write $W_r=\AS_{\rmdp_r}(\parity(\coloring_r))\cap U_r$ for its
almost-sure winning region and $\agentpol_r$ for a corresponding
positional almost-sure winning policy.

\subsection{Transfer between the RMDP and its trap value-class RMDPs}
The trap value-class RMDP $\rmdp_r$ lives over $U_r\cup\{\bot\}$, so its
one-step choices are distributions over $U_r\cup\{\bot\}$ and not
elements of $\uncert(s,a)\subseteq\Delta(\states)$. The following lemma
does the translation once and for all. It is used at every point below
where a play of $\rmdp$ is recognised as a play of $\rmdp_r$ or
conversely. Recall the collapse map
\[
  \kappa_r:\Delta(\states)\to\Delta(U_r\cup\{\bot\}),\qquad
  \kappa_r(p)(t)=p(t)\ \ (t\in U_r),\qquad
  \kappa_r(p)(\bot)=p(\states\setminus U_r),
\]
and that $\uncert_r(s,a)=\kappa_r(\tightset{\val^{\agentpol}}{s}{a})$
for $s\in U_r$ and $a$ tight at $s$.

\begin{lemma}[Transfer]\label{lem:transfer}
Let $\uvector=\val^{\agentpol}$, let $r<1$ with $U_r\neq\emptyset$, let
$s\in U_r$ and let $a\in\tightacts{\uvector}{s}$. Then:
\begin{enumerate}
\item \emph{(Push-forward.)} For every
  $p\in\tightset{\uvector}{s}{a}$ we have
  $\kappa_r(p)\in\uncert_r(s,a)$.
\item \emph{(Pull-back.)} For every $q\in\uncert_r(s,a)$ there is
  $p\in\tightset{\uvector}{s}{a}$ with $\kappa_r(p)=q$. Such a $p$ is a
  legal environment choice of $\rmdp$ at $(s,a)$ and satisfies
  $p\cdot\uvector=\uvector_s$.
\item \emph{(Plays.)} Let $\agentpol'$ be a positional agent policy of
  $\rmdp$ that plays tight actions at every state of $U_r$, and let
  $X\subseteq U_r$. Suppose $\envpol$ is a positional environment policy
  of $\rmdp$ such that $\envpol(t,\agentpol'(t))$ lies in the tight face
  at $(t,\agentpol'(t))$ and is supported in $X$.
  Then $\kappa_r\circ\envpol$ restricted to $X$ is a legal environment
  policy of $\rmdp_r$, $\agentpol'$ restricted to $X$ is a legal agent
  policy of $\rmdp_r$, and for every $x\in X$ the two induced measures
  on plays coincide: the play of $\rmdp$ from $x$ never leaves $X$ and
  the corresponding play of $\rmdp_r$ never visits $\bot$.
\end{enumerate}
\end{lemma}

\begin{proof}[Proof of Lemma~\ref{lem:transfer}]
(1) is the definition of $\uncert_r(s,a)$
as the image of the tight face. (2) holds because a set-image consists
exactly of the points having a preimage. The two properties of $p$ are
$\tightset{\uvector}{s}{a}\subseteq\uncert(s,a)$ and the definition of
the tight face together with tightness of $a$. For~(3), we check the three assertions in turn. The agent's actions are
tight, hence available in $\rmdp_r$. The environment's choices lie in
$\uncert_r$ by~(2). These choices are supported in $X\subseteq U_r$, so
$\kappa_r(\envpol(t,\agentpol'(t)))=\envpol(t,\agentpol'(t))$ by~(1).
The two transition kernels therefore agree on $X$, and neither play ever
leaves $X$. Equal kernels on $X$ give equal
measures on cylinders, hence equal measures on plays.
\end{proof}

Note that, by part~(4), on plays that stay in $U_r$ the models $\rmdp$
and $\rmdp_r$ are literally the same Markov chain, and $\coloring_r$
agrees with $\coloring$ there. So, such a play satisfies $\spec$ if and
only if it satisfies $\parity(\coloring_r)$. This is the only fact about
the collapse that the correctness proofs use.

\subsection{Markov-chain lemmas}

\begin{lemma}[(Sub/super)martingale decomposition]
\label{lem:mart}
Let $(\agentpol,\envpol)$ be positional, inducing a finite Markov chain
with matrix $P$, and let $\uvector\in[0,1]^{\states}$.
\begin{enumerate}
\item If $(P\uvector)(s)\ge\uvector(s)$ for all $s$, then $\uvector$ is
  constant on every closed recurrent class $\recurrent$. Write
  $\uvector_{\recurrent}$ for this value. Moreover,
  $\uvector(s)\le\sum_{\recurrent}
  \Pr_s[\lozenge\recurrent]\,\uvector_{\recurrent}$ for all $s$.
\item Symmetrically, if $(P\uvector)(s)\le\uvector(s)$ for all $s$,
  then $\uvector$ is constant on every closed recurrent class and
  $\uvector(s)\ge\sum_{\recurrent}
  \Pr_s[\lozenge\recurrent]\,\uvector_{\recurrent}$.
\end{enumerate}
Moreover, in either case the defining inequality is an \emph{equality}
at every state of every closed recurrent class.
\end{lemma}

\begin{proof}[Proof of Lemma~\ref{lem:mart}]
We prove (1); (2) is symmetric. Let $\recurrent$ be a closed recurrent
class with stationary distribution $\mu$, so $\mu>0$ on $\recurrent$ and
$\mu$ is supported on $\recurrent$. Since $\recurrent$ is closed,
$\sum_{x\in\recurrent}\mu(x)(P\uvector)(x)
=\sum_{y\in\recurrent}\uvector(y)\sum_{x\in\recurrent}\mu(x)P(x,y)
=\sum_{y\in\recurrent}\mu(y)\uvector(y)$ by stationarity. Comparing with
the pointwise inequality $(P\uvector)(x)\ge\uvector(x)$ and using
$\mu>0$, every such inequality is an \emph{equality} on $\recurrent$;
thus $\uvector$ restricted to $\recurrent$ is harmonic for an
irreducible finite chain and is therefore constant. For the inequality,
$\uvector(X_n)$ is a bounded submartingale, so
$\uvector(s)\le\mathbb{E}[\uvector(X_n)]$ for all $n$. The play is
absorbed in some $\recurrent$ almost surely and
$\uvector\equiv\uvector_{\recurrent}$ there, so
$\uvector(X_n)\to\uvector_{\recurrent(\pi)}$ almost surely, and bounded
convergence gives $\uvector(s)\le\sum_{\recurrent}
\Pr_s[\lozenge\recurrent]\,\uvector_{\recurrent}$.
\end{proof}

The next lemma states that a policy that is almost-sure winning for a
prefix-independent objective on the whole almost-sure region can never
be forced to leave it.

\begin{lemma}[Invariance of almost-sure regions]\label{lem:invariance}
Let $\rmdp'$ be an RMDP, let $W=\AS_{\rmdp'}(\spec')$ for a
prefix-independent objective $\spec'$, and let $\agentpol_W$ be an agent
policy that is almost-sure winning from every state of $W$. Then for
every $w\in W$ and every environment policy $\envpol$,
$\Pr^{\agentpol_W,\envpol}_w[\forall n:X_n\in W]=1$.
\end{lemma}

\begin{proof}[Proof of Lemma~\ref{lem:invariance}]
Suppose towards a contradiction that under some $\envpol$ the play from
$w$ leaves $W$ with positive probability, and fix a history $h$ of
positive probability ending in a state $s'\notin W$. Let $\envpol'$ be
any environment policy and let $\envpol''$ be the environment policy
that behaves as $\envpol$ along $h$ and as $\envpol'$ afterwards. Then
$h$ still has positive probability under
$(\agentpol_W,\envpol'')$. Since $\agentpol_W$ is almost-sure winning
from $w$, $\Pr^{\agentpol_W,\envpol''}_w[\spec']=1$, and conditioning on
the positive-probability event $h$ and using prefix-independence of
$\spec'$ gives
$\Pr^{(\agentpol_W)|_h,\envpol'}_{s'}[\spec']=1$. As $\envpol'$ was
arbitrary, the shifted policy $(\agentpol_W)|_h$ is almost-sure winning
from $s'$, i.e.\ $s'\in W$, a contradiction.
\end{proof}

Applied to the trap value-class RMDPs $\rmdp_r$, Lemma~\ref{lem:invariance}
yields the property of $W_r$ announced in Section~\ref{sec:rmdp} and
used throughout the correctness proofs below: the environment choices
available in $\rmdp_r$ cannot push the play out of $W_r$. Since
$\bot\notin W_r$, this single statement also says that the play cannot
leave the value class $U_r$.

\begin{lemma}[Support in $W_r$]\label{lem:as-facts}
Assume $I=\emptyset$, and let $r<1$ with $U_r\neq\emptyset$. Then
$\mathrm{supp}(q)\subseteq W_r$ for every $s\in W_r$ and every
$q\in\uncert_r(s,\agentpol_r(s))$. In particular $q(\bot)=0$, so no
environment choice of $\rmdp_r$ available against $\agentpol_r$ leaves
$U_r$.
\end{lemma}

\begin{proof}[Proof of Lemma~\ref{lem:as-facts}]
The objective $\parity(\coloring_r)$ is prefix-independent and
$\agentpol_r$ is almost-sure winning from every state of $W_r\subseteq
\AS_{\rmdp_r}(\parity(\coloring_r))$, so Lemma~\ref{lem:invariance}
applied to $\rmdp_r$ gives
$\Pr^{\agentpol_r,\envpol}_s[\forall n:X_n\in W_r]=1$ for every
environment policy $\envpol$ of $\rmdp_r$. Were some
$q\in\uncert_r(s,\agentpol_r(s))$ to place positive mass outside $W_r$,
the environment policy playing $q$ at $s$ would leave $W_r$ with
positive probability.
\end{proof}

\begin{remark}\label{rem:tight-by-construction}
The policy $\agentpol_r$ plays only \emph{tight}
actions on $W_r$: the
only actions available at $s\in U_r$ in $\rmdp_r$ are the tight ones
$\tightacts{\val^{\agentpol}}{s}$, which form a nonempty set because
$\agentpol(s)$ is tight by \eqref{eq:rmc-bellman}.
\end{remark}

\subsection{Environment spoiling}

\begin{restatable}[Environment spoiling on on value-class RMDPs]{proposition}{spoilprop}
\label{prop:spoil}
Let $r<1$ with $U_r\neq\emptyset$ and suppose
$W_r=\AS_{\rmdp_r}(\parity(\coloring_r))=\emptyset$. Then the
environment has a positional policy $\envpol_r$ on $\rmdp_r$ such that
$\Pr^{\agentpol'',\envpol_r}_s[\parity(\coloring_r)]=0$ for every
agent policy $\agentpol''$ and every state $s$.
\end{restatable}

\begin{proof}[Proof of Proposition~\ref{prop:spoil}]
Write $\spec_r:=\parity(\coloring_r)$.
The instance $\rmdp_r$ is a linearly defined RMDP: at $s\in U_r$ and $a$
tight, $\uncert_r(s,a)=\kappa_r(\tightset{\val^{\agentpol}}{s}{a})$ is
the image under a linear map of the face of $\uncert(s,a)$ cut out by
one additional linear inequality, hence a nonempty polytope, presented
as the projection of a linear system in the variables $(q,p)$. Also,
$\uncert_r(\bot,a)=\{\delta_{\bot}\}$. Proposition~\ref{prop:polytopic} requires only that the uncertainty sets
be nonempty bounded polytopes, and hence applies to $\rmdp_r$ with
objective $\spec_r$. It yields that there exists a positional agent policy
$\agentpol^\ast$ and a positional environment policy $\envpol_r$ such that
$\Pr^{\agentpol^\ast,\envpol_r}_s[\spec_r]=\val^{\rmdp_r}_s(\spec_r)$.
We claim $\val^{\rmdp_r}_s(\spec_r)=0$ for every $s$. Suppose not, and
fix $s_0$ with $\val^{\rmdp_r}_{s_0}(\spec_r)>0$. Being positional, $(\agentpol^\ast,\envpol_r)$ induces a
finite Markov chain over $U_r\cup\{\bot\}$, so by
\eqref{eq:absorption} applied with the coloring $\coloring_r$,
\[
  \val^{\rmdp_r}_s(\spec_r)=\Pr^{\agentpol^\ast,\envpol_r}_s[\spec_r]
  =\sum_{\recurrent\ \text{even}}\Pr_s[\lozenge\recurrent],
\]
where ``even'' refers to $\max\coloring_r(\recurrent)$.
As $\val^{\rmdp_r}_{s_0}(\spec_r)>0$, some even class $\recurrent$ is
reached with positive probability. Fix $x\in\recurrent$. From $x$ the
chain stays in $\recurrent$ almost surely, so
$\val^{\rmdp_r}_x(\spec_r)=\Pr^{\agentpol^\ast,\envpol_r}_x[\spec_r]=1$.
But $\agentpol^\ast$ secures the value, so
$\inf_{\envpol}\Pr^{\agentpol^\ast,\envpol}_x[\spec_r]
=\val^{\rmdp_r}_x(\spec_r)=1$, i.e.
$\Pr^{\agentpol^\ast,\envpol}_x[\spec_r]=1$ for \emph{every}
environment policy $\envpol$ of $\rmdp_r$. Thus $\agentpol^\ast$ is
almost-sure winning from $x$ in $\rmdp_r$, so
$x\in\AS_{\rmdp_r}(\spec_r)$, contradicting
$\AS_{\rmdp_r}(\spec_r)=\emptyset$.

Therefore $\val^{\rmdp_r}_s(\spec_r)=0$ for all $s$, and since
$\envpol_r$ attains the value, $\Pr^{\agentpol'',\envpol_r}_s[\parity(\coloring_r)]=0$ for every
agent policy $\agentpol''$ and every state $s$.
\end{proof}

\subsection{Correctness of the improvement steps}


\begin{proof}[Proof of Lemma~\ref{lem:value}]
Write $\uvector:=\val^{\agentpol}$. Fix a positional environment policy
$\envpol$ and let $P$ be the matrix of the chain induced by
$(\agentpol',\envpol)$. For $s\notin I$ we have
$\agentpol'(s)=\agentpol(s)$ and, by \eqref{eq:rmc-bellman},
$(P\uvector)(s)=\envpol(s,\agentpol(s))\cdot\uvector
\ge\min_{p\in\uncert(s,\agentpol(s))}p\cdot\uvector=\uvector_s$. For
$s\in I$, $(P\uvector)(s)\ge\min_{p\in\uncert(s,\agentpol'(s))}
p\cdot\uvector=\nu(s,\agentpol'(s))=\max_a\nu(s,a)>\uvector_s$, since
$\agentpol'(s)$ maximises $\nu(s,\cdot)$. Hence $\uvector$ is a
submartingale for $P$, and by Lemma~\ref{lem:mart}(1) it is constant on
every closed recurrent class $\recurrent$ with
$\uvector(s)\le\sum_{\recurrent}
\Pr_s[\lozenge\recurrent]\,\uvector_{\recurrent}$.

First, no closed recurrent class meets $I$: on $\recurrent$ the
submartingale inequalities are equalities (Lemma~\ref{lem:mart}),
whereas at states of $I$ the inequality is strict. Hence
$\agentpol'=\agentpol$ on every closed recurrent class $\recurrent$.

Second, every closed recurrent class $\recurrent$ with
$\uvector_{\recurrent}>0$ is even. Let $x\in\recurrent$ and let
$\widetilde{\envpol}$ be any environment policy of $\rmdp^{\agentpol}$
that agrees with $\envpol$ on $\recurrent$. Since $\recurrent$ is
closed under the kernels of $(\agentpol,\envpol)$ and
$\agentpol'=\agentpol$ there, the play from $x$ under
$(\agentpol,\widetilde{\envpol})$ stays in $\recurrent$ forever and
satisfies $\spec$ with probability $1$ if $\recurrent$ is even and $0$
if odd. But
$\Pr^{\agentpol,\widetilde{\envpol}}_x[\spec]\ge
\inf_{\envpol''}\Pr^{\agentpol,\envpol''}_x[\spec]
=\val^{\agentpol}_x=\uvector_{\recurrent}>0$, so $\recurrent$ is even.

Combining with \eqref{eq:absorption} and using that odd classes have
$\uvector_{\recurrent}=0$,
\[
  \Pr^{\agentpol',\envpol}_s[\spec]
  =\sum_{\recurrent\ \text{even}}\Pr_s[\lozenge\recurrent]
  \ge\sum_{\recurrent}\Pr_s[\lozenge\recurrent]\,\uvector_{\recurrent}
  \ge\uvector_s .
\]
Taking the infimum over positional $\envpol$, which by
Proposition~\ref{prop:polytopic} applied to the RMC
$\rmdp^{\agentpol'}$ computes $\val^{\agentpol'}$, gives
$\val^{\agentpol'}\ge\val^{\agentpol}$.

For strictness at $s\in I$: by \eqref{eq:rmc-bellman} applied to
$\rmdp^{\agentpol'}$ and monotonicity of
$\uvector\mapsto\min_p p\cdot\uvector$,
\begin{align*}
  \val^{\agentpol'}_s=\min_{p\in\uncert(s,\agentpol'(s))}
  p\cdot\val^{\agentpol'}
  \ \ge\ \min_{p\in\uncert(s,\agentpol'(s))}p\cdot\val^{\agentpol}\\
  =\max_{a}\min_{p\in\uncert(s,a)}p\cdot\val^{\agentpol}
  >\val^{\agentpol}_s,
\end{align*}
where the last equality uses that $\agentpol'(s)$ attains the outer
maximum defining $I$.
\end{proof}


\begin{proof}[Proof of Lemma~\ref{lem:qual}]
Write $\uvector:=\val^{\agentpol}$. Since $I=\emptyset$,
\begin{equation}\label{eq:fix-i}
  \uvector_s=\max_{a}\min_{p\in\uncert(s,a)}p\cdot\uvector
  \qquad\text{for all }s,
\end{equation}
and the action $\agentpol(s)$ attains it by \eqref{eq:rmc-bellman};
hence every action played by $\agentpol'$ is tight: $\agentpol(s)$ is
tight at $s$ for every $s$, and $\agentpol_r(s)$ is tight at $s$ for
every $s\in W_r$ by construction of $\rmdp_r$
(Remark~\ref{rem:tight-by-construction}).

\emph{Soundness.} Fix a positional environment policy $\envpol$ of
$\rmdp$ with matrix $P$ (for policy $\agentpol'$). Tightness gives
$(P\uvector)(s)\ge\uvector_s$ for all $s$, so Lemma~\ref{lem:mart}(1)
applies. It suffices to show that every closed recurrent class
$\recurrent$ with $\uvector_{\recurrent}>0$ is even. The bound
$\Pr^{\agentpol',\envpol}_s[\spec]\ge\uvector_s$ then follows exactly as
in Lemma~\ref{lem:value}, and taking the infimum over positional
$\envpol$ gives $\val^{\agentpol'}\ge\val^{\agentpol}$.

Let $\recurrent$ be a closed recurrent class and put
$r:=\uvector_{\recurrent}$. Since $\uvector\equiv r$ on $\recurrent$ we
have $\recurrent\subseteq U_r$, so $\coloring_r$ agrees with $\coloring$
on $\Inf(\pi)=\recurrent$. Moreover, by the last clause of Lemma~\ref{lem:mart} the submartingale
inequality is an equality at every $x\in\recurrent$, that is,
$\envpol(x,\agentpol'(x))\cdot\uvector=\uvector_x$. Since
$\agentpol'(x)$ is tight, this places $\envpol(x,\agentpol'(x))$ in the
tight face $\tightset{\uvector}{x}{\agentpol'(x)}$. It is moreover
supported in $\recurrent$, as $\recurrent$ is closed. Hence
Lemma~\ref{lem:transfer}(3), applied with $X=\recurrent$, identifies the
play of $\rmdp$ from any $x\in\recurrent$ under $(\agentpol',\envpol)$
with a play of $\rmdp_r$ under
$(\agentpol',\kappa_r\circ\envpol)$ that never visits $\bot$. Two cases.

\emph{Case 1: $\recurrent\cap W_r=\emptyset$.} Then
$\agentpol'=\agentpol$ on $\recurrent$, and the argument of
Lemma~\ref{lem:value} shows that if $r>0$ then $\recurrent$ is even.

\emph{Case 2: $\recurrent\cap W_r\neq\emptyset$.} Let
$x\in\recurrent\cap W_r$. At states of $W_r$ the policy $\agentpol'$
plays $\agentpol_r$, and by the previous paragraph the environment's
choices on $\recurrent$ are, after $\kappa_r$, choices of $\rmdp_r$;
hence Lemma~\ref{lem:as-facts} applies and every successor of a state of
$\recurrent\cap W_r$ lies in $W_r$, while lying in $\recurrent$ because
$\recurrent$ is closed. Thus the play from $x$ never leaves
$\recurrent\cap W_r$, and since $\recurrent$ is a recurrent class every
state of $\recurrent$ is visited almost surely, so
$\recurrent\subseteq W_r$ and $\agentpol'$ follows $\agentpol_r$
throughout $\recurrent$. Consequently the play from $x$ is a play of $\rmdp_r$ consistent with
$\agentpol_r$, started in
$x\in\AS_{\rmdp_r}(\parity(\coloring_r))$. It therefore satisfies
$\parity(\coloring_r)$ almost surely. Now $\Inf(\pi)=\recurrent$ almost
surely, and $\coloring_r=\coloring$ on $\recurrent$. Hence
$\max\coloring(\recurrent)$ is even, i.e.\ $\recurrent$ is even.

This
proves the claim, and exactly as in Lemma~\ref{lem:value} we conclude
$\val^{\agentpol'}\ge\val^{\agentpol}$.

\emph{Strictness on $W$.} Let $x\in W_r$, so $\uvector_x=r<1$. By
Proposition~\ref{prop:polytopic} applied to the RMC
$\rmdp^{\agentpol'}$, fix a positional environment policy
$\envpol^\ast$ attaining $\val^{\agentpol'}$ simultaneously at all
states, i.e.\
$\val^{\agentpol'}_t=\Pr^{\agentpol',\envpol^\ast}_t[\spec]$ for all
$t$. Call a state $t$ \emph{loose} if
$\envpol^\ast(t,\agentpol'(t))\cdot\uvector>\uvector_t$ and let $L$ be the set of
loose states. Observe that since $\agentpol'(t)$ is tight, a state $t$ is \emph{loose} if
$\envpol^\ast(t,\agentpol'(t))$ lies off the tight face
$\tightset{\uvector}{t}{\agentpol'(t)}$. Let $T$ be the first time the play started at $x$ visits
$L$. Note that $\{T=\infty\}$ is exactly the event that the play never
visits $L$.

We first show that $\Pr[\spec\mid T=\infty]=1$ whenever
$\Pr[T=\infty]>0$. Before the play visits $L$, the environment's choices
lie on tight faces, so their $\kappa_r$-images are available in
$\rmdp_r$ by Lemma~\ref{lem:transfer}(2). Since
$\agentpol'=\agentpol_r$ on $W_r$, Lemma~\ref{lem:as-facts} keeps the
play inside $W_r\subseteq U_r$. In particular $X_n\in W_r$ for every
$n\le T$, and for every $n$ when $T=\infty$. Define an environment policy
$\widetilde{\envpol}$ of $\rmdp_r$ by
$\widetilde{\envpol}(t,a):=\kappa_r(\envpol^\ast(t,a))$ whenever $t\in
W_r\setminus L$ and $a$ is tight at $t$ (legal precisely because $t$
is not loose) and by an arbitrary element of $\uncert_r(t,a)$
otherwise. By Lemma~\ref{lem:transfer}(3) the chains induced by
$(\agentpol',\envpol^\ast)$ in $\rmdp$ and by
$(\agentpol',\widetilde{\envpol})$ in $\rmdp_r$ assign the same measure
to every set of paths from $x$ avoiding $L$, since their kernels agree
at every state of $W_r\setminus L$. Under
$(\agentpol_r,\widetilde{\envpol})$ the play from $x$ stays in $W_r$ and
satisfies $\parity(\coloring_r)$ almost surely, hence also $\spec$
almost surely, as it never
visits $\bot$ and $\coloring_r=\coloring$ on $U_r$. Therefore
$\Pr^{\agentpol',\envpol^\ast}_x[\{T=\infty\}\cap\neg\spec]
=\Pr^{\agentpol',\widetilde{\envpol}}_x[\{T=\infty\}\cap\neg\spec]=0$,
which is the claim.

If $L=\emptyset$ then $T=\infty$ almost surely and
$\val^{\agentpol'}_x=\Pr^{\agentpol',\envpol^\ast}_x[\spec]=1>r$, as
required. Otherwise set
\[
  \varepsilon:=\min\Big\{\,\envpol^\ast(t,\agentpol'(t))\cdot\uvector
  -\uvector_t\ \Big|\ t\in L\,\Big\}\;>\;0,
\]
a minimum over at most $|\states|$ states. On $\{T<\infty\}$ we have $X_T\in W_r\cap L$, hence
$\uvector(X_T)=r$, and the step taken at time $T$ gives
\[
  \mathbb{E}[\uvector(X_{T+1})\mid\,T<\infty]
  =\envpol^\ast(X_T,\agentpol'(X_T))\cdot\uvector\ \ge\ r+\varepsilon .
\]
By the strong Markov property, prefix-independence of $\spec$, and
$\Pr^{\agentpol',\envpol^\ast}_t[\spec]=\val^{\agentpol'}_t
\ge\val^{\agentpol}_t=\uvector_t$ (soundness) for the positional
continuation from any state $t$,
\[
  \Pr[\spec\mid T<\infty]
  =\mathbb{E}\big[\val^{\agentpol'}(X_{T+1})\mid T<\infty\big]
  \ \ge\ \mathbb{E}\big[\uvector(X_{T+1})\mid T<\infty\big]
  \ \ge\ r+\varepsilon.
\]
Combining, since
$\val^{\agentpol'}_x=\Pr^{\agentpol',\envpol^\ast}_x[\spec]$,
\[
  \val^{\agentpol'}_x
  \ \ge\ \Pr[T=\infty]\cdot1+\Pr[T<\infty]\cdot(r+\varepsilon)
  \ \ge\ \min\{1,\,r+\varepsilon\}\ >\ r=\uvector_x ,
\]
the last inequality because $r<1$.
\end{proof}

\subsection{Optimality of fixpoints, termination, and complexity}

\begin{restatable}[Optimality of fixpoints]{theorem}{fixpointthm}
\label{thm:fixpoint}
Let $\agentpol$ be a positional agent policy whose value vector
$\val^{\agentpol}$ satisfies
\begin{enumerate}
\item[(i)] $\val^{\agentpol}_s=\max_{a}\min_{p\in\uncert(s,a)}
  p\cdot\val^{\agentpol}$ for all $s$ (equivalently $I=\emptyset$), and
\item[(ii)] $W_r=\AS_{\rmdp_r}(\parity(\coloring_r))=\emptyset$ for
  every value $r<1$ of $\val^{\agentpol}$.
\end{enumerate}
Then $\val^{\agentpol}=\val^{\rmdp}(\spec)$ and $\agentpol$ is optimal.
Moreover an optimal positional environment policy exists, obtained by
combining, on each value class of value $<1$, the environment's
spoiling policies of the value-class RMDPs
(Proposition~\ref{prop:spoil}) with value-minimising choices elsewhere.
\end{restatable}

\begin{proof}[Proof of Theorem~\ref{thm:fixpoint}]
Write $\uvector:=\val^{\agentpol}$. Since $\uvector$ is guaranteed by
$\agentpol$ against every environment policy, $\val^{\rmdp}(\spec)\ge
\uvector$ pointwise, it remains to find a single environment policy
$\widehat{\envpol}$ with
\begin{equation}\label{eq:cap}
  \Pr^{\agentpol',\widehat{\envpol}}_s[\spec]\le\uvector_s
  \qquad\text{for every agent policy }\agentpol'\text{ and every }s.
\end{equation}
Indeed, \eqref{eq:cap} gives $\val^{\rmdp}_s(\spec)\le\uvector_s=
\val^{\agentpol}_s\le\val^{\rmdp}_s(\spec)$, so
$\val^{\agentpol}=\val^{\rmdp}(\spec)$, the policy $\agentpol$ is
optimal (it attains $\uvector$), and $\widehat{\envpol}$ attains
$\inf_{\envpol}\sup_{\agentpol'}$ and is an optimal environment policy
by Proposition~\ref{prop:polytopic}.

\emph{Construction of $\widehat{\envpol}$.} By (ii) and
Proposition~\ref{prop:spoil}, for every value $r<1$ of $\uvector$ the
environment has a positional policy $\envpol_r$ \emph{on $\rmdp_r$}
that spoils the trap objective uniformly:
$\Pr^{\agentpol'',\envpol_r}_s[\parity(\coloring_r)]=0$ for every agent
policy $\agentpol''$ and every state $s$ of $\rmdp_r$. Equivalently,
against $\envpol_r$ the play in $\rmdp_r$ almost surely either reaches
$\bot$ or remains in $U_r$ forever
and violates the original parity condition there. Define the positional
environment policy $\widehat{\envpol}$ of $\rmdp$ at each pair $(s,a)$:
\begin{itemize}
\item if $s\in U_r$ for some $r<1$ and $a$ is tight at $s$: by
      Lemma~\ref{lem:transfer}(2) pick
      $p_{s,a}\in\tightset{\uvector}{s}{a}$ with
      $\kappa_r(p_{s,a})=\envpol_r(s,a)$ and let
      $\widehat{\envpol}(s,a):=p_{s,a}$. This is the \emph{uncollapsed}
      form of the choice prescribed by $\envpol_r$, and it is a legal
      environment choice of $\rmdp$;
\item otherwise ($a$ not tight at $s$, or $\uvector_s=1$): let
      $\widehat{\envpol}(s,a)\in\arg\min_{p\in\uncert(s,a)}
      p\cdot\uvector$, which is nonempty because $\uncert(s,a)$ is a
      nonempty compact polytope.
\end{itemize}
In all cases $\widehat{\envpol}(s,a)\cdot\uvector\le\uvector_s$: for a
tight action, by definition,
$p_{s,a}\cdot\uvector=\uvector_s$, the dot product being taken over the
\emph{original} vector $\uvector$, i.e.\ before the collapse and hence
including the states outside $U_r$. For a non-tight action
$\min_{p}p\cdot\uvector<\uvector_s$ by (i). Finally, if $\uvector_s=1$ then
trivially $p\cdot\uvector\le1$. 

\emph{Proof of \eqref{eq:cap}.} Fix a starting state $s$. With
$\widehat{\envpol}$ fixed and positional, the agent faces an ordinary
MDP, which by Remark~\ref{rem:app:rmc-mdp} admits a positional optimal
agent policy for the parity objective. So, it suffices to bound
$\Pr^{\agentpol',\widehat{\envpol}}_s[\spec]$ for positional
$\agentpol'$. Let $P$ be the matrix of the finite Markov chain induced
by $(\agentpol',\widehat{\envpol})$. By what was discussed before and the definition of $\widehat{\envpol}$,
$(P\uvector)(t)\le\uvector_t$ for all $t$, so $\uvector$ is a
supermartingale and by Lemma~\ref{lem:mart}(2) it is constant on every
closed recurrent class $\recurrent$, with
$\uvector_s\ge\sum_{\recurrent}
\Pr_s[\lozenge\recurrent]\,\uvector_{\recurrent}$.

\emph{Claim: every even closed recurrent class $\recurrent$ has
$\uvector_{\recurrent}=1$.} Suppose instead $\uvector_{\recurrent}=r<1$;
then $\recurrent\subseteq U_r$. On $\recurrent$ the supermartingale
inequalities are equalities, so for every $x\in\recurrent$ the action
$a=\agentpol'(x)$ is tight (a non-tight action gives a strict decrease)
and hence $\widehat{\envpol}(x,a)=p_{x,a}$ is the uncollapsed tight-face
choice prescribed by $\envpol_r$. As $\recurrent$ is closed, these
choices are supported in $\recurrent\subseteq U_r$, so
Lemma~\ref{lem:transfer}(3) with $X=\recurrent$ identifies the play from
$x$ under $(\agentpol',\widehat{\envpol})$, which stays in
$\recurrent$ forever with probability $1$, with a play of $\rmdp_r$
under $\envpol_r$ that never reaches $\bot$. By the spoiling guarantee
of $\envpol_r$ such a play violates $\parity(\coloring_r)$ almost
surely, and since it never leaves $U_r$, where $\coloring_r=\coloring$,
it violates $\spec$ almost surely. As $\Inf(\pi)=\recurrent$ almost
surely, $\max\coloring(\recurrent)$ is odd, contradicting that
$\recurrent$ is even. This proves the claim.

By \eqref{eq:absorption} and the claim,
\[
  \Pr^{\agentpol',\widehat{\envpol}}_s[\spec]
  =\sum_{\recurrent\text{ even}}\Pr_s[\lozenge\recurrent]
  =\sum_{\recurrent\text{ even}}\Pr_s[\lozenge\recurrent]\,
    \uvector_{\recurrent}
  \le\sum_{\recurrent}\Pr_s[\lozenge\recurrent]\,\uvector_{\recurrent}
  \le\uvector_s,
\]
which is \eqref{eq:cap}. The optimal environment policy
$\widehat{\envpol}$ is positional and has exactly the announced
structure: uncollapsed tight-face (spoiling) choices inside each value
class of value $<1$, and value-minimising choices elsewhere.
\end{proof}


\begin{proof}[Proof of Theorem~\ref{cor:total}]
If $\ProfSwitch(\rmdp,\agentpol)\neq\agentpol$, then by
Lemma~\ref{lem:value} or Lemma~\ref{lem:qual} we have
$\val^{\agentpol'}\ge\val^{\agentpol}$ with strict inequality in at
least one coordinate. The value vectors are thus strictly increasing in
the product order along the run, so no positional policy is visited
twice and the loop terminates within the number of positional agent
policies, $|\actions|^{|\states|}$. On termination $\agentpol$ is a
fixpoint: $I=\emptyset$ and $W_r=\emptyset$ for every value $r<1$, the conditions of Theorem~\ref{thm:fixpoint}. By Theorem~\ref{thm:fixpoint} the
returned policy is optimal and
$\val^{\agentpol}=\val^{\rmdp}(\spec)$.
\end{proof}

\begin{restatable}[Complexity]{theorem}{complexitythm}
\label{thm:complexity}
$\StratImp$ performs at most $|\actions|^{|\states|}$ iterations. One
iteration solves $O(|\states|^3+|\states|\cdot|\actions|)$ linear
programs, each of size polynomial in $|\rmdp|$, performs at most
$|\states|$ almost-sure parity computations on polytopic RMDPs of
description size $O(|\rmdp|)$, and performs
$O(|\states|^2\cdot|\actions|)$ additional bookkeeping. Consequently,
writing $T_{\mathrm{AS}}(n)$ for the cost of one almost-sure parity
computation on an instance of size $n$, one iteration runs in time
$\mathrm{poly}(|\rmdp|)+|\states|\cdot
T_{\mathrm{AS}}(O(|\rmdp|))$, and the whole algorithm in time
$|\actions|^{|\states|}\cdot\big(\mathrm{poly}(|\rmdp|)
+|\states|\cdot T_{\mathrm{AS}}(O(|\rmdp|))\big)$.
\end{restatable}

\begin{proof}[Proof of Theorem~\ref{thm:complexity}]
The iteration bound is Theorem~\ref{cor:total}. Consider one call to
$\ProfSwitch$.

Line~1 computes the parity value vector of the linearly defined RMC
$\rmdp^{\agentpol}$. By Theorem~\ref{thm:rmc:parity-poly} and its proof
this takes $O(|\states|^3)$ linear programs: computing the almost-sure
violation set $W_{\mathrm{odd}}=\asoddsat(\states)$ performs
$O(|\states|^3)$ feasibility and maximisation programs
(Lemma~\ref{thm:rmc:almost-sure-parity}), and the safety
program~$(\star)$, which has $O(|\states|^2)$ variables and
polynomially many constraints, is solved once per source state. Note
that the cubic bound of Lemma~\ref{thm:rmc:almost-sure-parity} covers
the \emph{whole} nested recursion of $\asoddsat$ and $\mecdecomp$, and
is not the cost of one MEC decomposition multiplied by the number of
$\asoddsat$ nodes: taking the two recursions together, there are at
most $|\states|$ levels, the sets processed at any fixed level are
pairwise disjoint, and one invocation of $\mecdecomp(B)$ costs only
$O(|B|^2)$ linear programs
(Theorem~\ref{thm:rmc:mec-decomp}), so each level costs
$O(|\states|^2)$ programs and the total is $O(|\states|^3)$ rather than
$O(|\states|^4)$, See the proof of
Lemma~\ref{thm:rmc:almost-sure-parity} in
Appendix~\ref{app:rmc:parity}. Lines~2--4 solve exactly
$|\states|\cdot|\actions|$ linear programs, one per pair $(s,a)$, each
of them a minimisation of a linear functional over the facet description
of $\uncert(s,a)$. Line~5 and the switch of line~8 are
$O(|\states|\cdot|\actions|)$ comparisons. This accounts for the
$O(|\states|^3+|\states|\cdot|\actions|)$ programs of the statement;
each has size polynomial in $|\rmdp|$ and is therefore solvable in time
polynomial in $|\rmdp|$.

The loop of lines~11--15 is entered at most $|\states|$ times, once per
nonempty value class of value $<1$. Each iteration of the loop assembles
$\rmdp_r$ and $\coloring_r$ and performs one almost-sure parity
computation. The assembly needs no further optimisation. Indeed
$\rmdp_r$ has state set $U_r\cup\{\bot\}$. The actions available at
$s\in U_r$ are the tight ones, obtained by comparing the already
computed $\nu(s,a)$ with $\val^{\agentpol}_s$. The colors are inherited
from $\coloring$ on $U_r$, with the single fresh odd color
$2\lceil d/2\rceil+1$ at $\bot$ and each uncertainty set is
\[
  \uncert_r(s,a)=\Big\{\,q\ \Big|\ \exists p:\ A_{s,a}\,p\le b_{s,a},
  \ \ p\cdot\val^{\agentpol}\le\val^{\agentpol}_s,\ \
  q=\kappa_r(p)\,\Big\},
\]
a linear system in $(q,p)$ built from the description of $\uncert(s,a)$
by adding one inequality and the defining equalities of $\kappa_r$. Note
that the optima $\nu(s,a)$ enter only through the tightness test: the
uncertainty sets themselves are described using $\val^{\agentpol}$
alone.

Two consequences for the size. First, each $\rmdp_r$ has description
size $O(|\rmdp|)$ and is written down in time $O(|\rmdp|)$. Second, and
more sharply, the value classes are pairwise disjoint and $\rmdp_r$
contains only the states of $U_r$, so
$\sum_{r<1}|\rmdp_r|=O(|\rmdp|)$: the instances handed to the
almost-sure parity procedure during one iteration have \emph{total}
description size $O(|\rmdp|)$. Hence the cost of the loop is
$\sum_{r<1}T_{\mathrm{AS}}(|\rmdp_r|)\le
|\states|\cdot T_{\mathrm{AS}}(O(|\rmdp|))$ plus
$O(|\rmdp|)$ bookkeeping. Grouping the identification of the value
classes, the patching of line~17 and the comparisons above gives the
stated $O(|\states|^2\cdot|\actions|)$ bookkeeping, and multiplying by
the iteration bound gives the overall running time.
\end{proof}

\begin{remark}[Space of the almost-sure subroutine]
\label{rem:as-space}
The almost-sure parity procedure of~\citet{AsadiCGKS26} runs in
polynomial space, and it is this property that underlies the third item
of Theorem~\ref{cor:total}. Concretely, the procedure is a recursion of
depth at most $|\states|$ that stores, at each frame, only a subset of
the state space together with $O(\log|\states|)$ counters. That is,
$O(|\states|)$ bits per frame and $O(|\states|^2)$ bits over the whole
stack, and whose only other work consists of feasibility and
optimisation queries against the linear descriptions of the uncertainty
sets. Each such query is a linear program of size polynomial in the
encoding of the instance, hence solvable in polynomial time and polynomial space, and the space it occupies is
released before the next query is issued. The recursion therefore uses
polynomial space in total, even though its running time is only
quasi-polynomial. 
\end{remark}

\begin{remark}\label{sec:rmdp:reach}
A reachability objective $\reach(\target)$ is the parity
objective of the coloring that assigns color $2$ to the
(absorbing) target states and color $1$ elsewhere, so our algorithm
applies unchanged. For this coloring the qualitative improvement never
happens: every value class with $r < 1$ contains no target state, so
every state of $U_r$ has odd color $1$ in $\coloring$ and every state
outside $U_r$ has the fresh dominating odd color in $\coloring_r$; hence
all states are odd-colored under $\coloring_r$ and
$\AS_{\rmdp_r}(\parity(\coloring_r)) = \emptyset$ automatically.
\end{remark}

\section{Additional Experimental Details}
\label{app:experiments}

\paragraph{Benchmarks.}
We evaluate both algorithms on three benchmarks. Garnet and Inventory
Management support both objectives; Frozen Lake supports only
reachability.

\begin{itemize}
\item \textbf{Garnet.}
The $N$ states are partitioned into $k+1$ disjoint subsets
$G_0,\dots,G_k$ (we fix $k=5$), with $G_0$ holding half of the states
and containing the initial state $s_0$. Every state has three actions,
each leading to $m=\lceil\sqrt{N}\rceil$ distinct successors drawn
uniformly at random without replacement. For a state in $G_i$ with
$i>0$, all successors are drawn from within $G_i$, so the subset is
never left. For a state in $G_0$, the successors are drawn from within
$G_0$, but each action additionally diverts a small random fraction of
its probability mass to a uniformly random state of a uniformly random
subset among $G_1,\dots,G_k$; this is how an action taken in $G_0$ can
leave $G_0$ and enter another subset. The nominal transition
distribution over the successors of every state--action pair is drawn
uniformly from the probability simplex (i.e., $\mathrm{Dirichlet}(1,\dots,1)$),
and is then wrapped in an $L_\infty$ or $L_1$ uncertainty ball of radius
$0.5/m$ or $0.1$, respectively.

For the reachability objective, each subset independently contains
one target state and one trap state, and the goal is to reach the
target state of some subset. For the parity objective, every state
is assigned a color (priority) drawn uniformly from $\{1,\dots,|G_i|\}$,
and the goal is to satisfy the induced parity condition. In the
scaling experiments the branching factor is coupled to the size,
$m=\lceil\sqrt{N}\rceil$, so that larger instances are also denser.

\item \textbf{Inventory Management.}
The state is the stock level $s\in\{0,\dots,N\}$ of a warehouse of
capacity $N$. At each step the agent orders $a\in\{0,\dots,a_{\max}\}$
new units, the demand $d\in\{0,\dots,d_{\max}\}$ materializes, and the
stock moves deterministically to $s'=s+a-d$. If the demand exceeds the
available stock ($d>s+a$), or the delivery pushes the stock over the
capacity ($s+a-d>N$), the run instead ends in one of two absorbing
failure states. We fix $d_{\max}=\lceil\sqrt{N}\rceil$ and
$a_{\max}=\lceil 0.75\,d_{\max}\rceil$, so the number of possible
demand values, and hence the branching factor, is
$m=d_{\max}+1=\Theta(\sqrt{N})$. The nominal demand distribution over
$\{0,\dots,d_{\max}\}$ is a lightly perturbed uniform (each outcome is
weighted by an independent $U[0.75,1.25]$ factor and the weights are
then normalized), drawn independently per state, and is wrapped in an
$L_\infty$ or $L_1$ uncertainty ball of radius $\delta=0.25/m$ or $0.1$,
respectively.

Both objectives use a threshold $K$ drawn from $U[0.35,0.65]\cdot N$.
Starting from $s_0=0$, the reachability objective is to bring the stock
up to $s\ge K$, and the parity objective is to keep the stock at
$s\ge K$ infinitely often. Both objectives also share a second
threshold $T$, drawn uniformly from $\{d_{\max},\dots,K\}$, which
governs a larger order option: once the stock reaches $s\ge T$, an
additional order of $d_{\max}$ units becomes available to the agent on
top of the amounts $\{0,\dots,a_{\max}\}$.

\item \textbf{Frozen Lake.}
The agent walks on a $k\times k$ grid from the top-left cell toward the
bottom-right target cell, so the state count is $N=k^2$. Each move goes
in the intended direction or in one of the two perpendicular directions
with probability $\tfrac13$ each, giving a branching factor of $m\le3$
(cut down at the grid's border). A random $20\%$ of the cells are
holes, which are absorbing failure states. The nominal move
distribution of each state is wrapped in an $L_\infty$ or $L_1$
uncertainty ball whose radius is drawn per state from $U[0,0.125]$ or
$U[0,0.25]$, respectively. Because the branching factor stays bounded
by $3$ regardless of $N$, this benchmark scales to much larger state
spaces than the other two. On Frozen Lake we experiment only with the
reachability objective (reaching the target cell).
\end{itemize}

\begin{remark}
The bounds of Theorem~\ref{thm:rmc:mec-decomp} and
Theorem~\ref{thm:complexity} are worst-case upper bounds on the
\emph{number} of primitive queries against the uncertainty sets, and on
all of our benchmarks the number actually issued is considerably smaller.
The gap is not incidental, and it is not a property of the particular
instances: it follows from the branching factor, and it applies to any
family of linearly defined RMDPs whose uncertainty sets are described
over few coordinates. We record it here because it is what reconciles the
cubic bound of Theorem~\ref{thm:complexity} with the measured solve times
of Figure~\ref{fig:scaling}.
\end{remark}

\begin{remark}\label{rem:app:iterations}
Figure~\ref{fig:iterations} records how many iterations our algorithm
and the baseline take on each benchmark. For both of them, the observed
counts remain well below the theoretical worst case, which is
exponential in the number of states. This agrees with what is commonly
reported for policy iteration on stochastic games, where the iteration
count is typically polynomial in practice even though no subexponential
bound is known. The practical behaviour of our algorithms is thus
considerably better than the worst-case analysis predicts.
\end{remark}

\begin{figure}[t]
    \centering
    \includegraphics[width=0.7\textwidth]{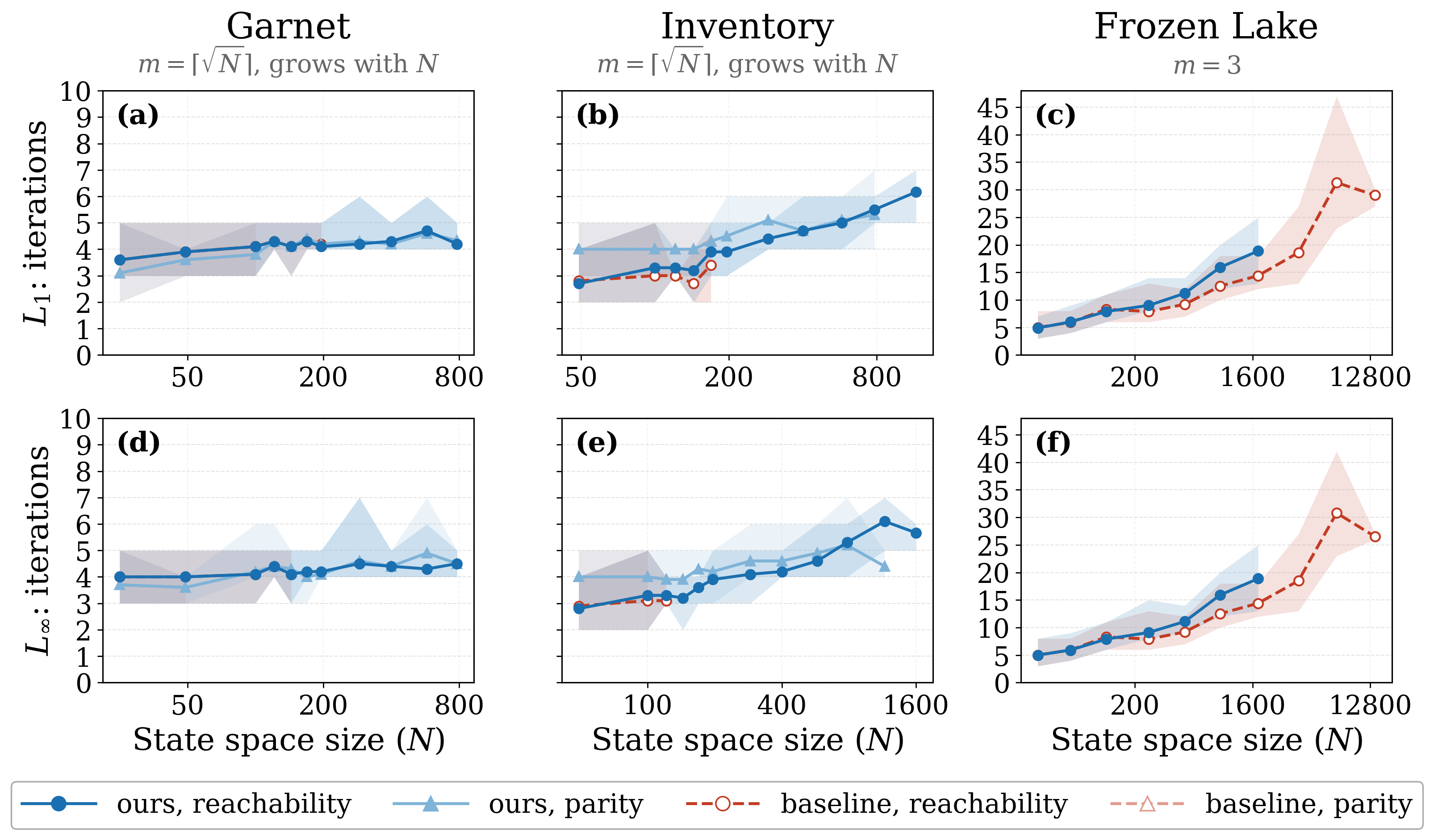}
    \caption{Comparison between number of iterations of the policy iteration algorithm on our benchmarks. The number of iterations is much smaller than the theoretical worst-case bound.}
    \label{fig:iterations}
\end{figure}
\end{document}